\documentclass{article}

\usepackage{microtype}
\usepackage{graphicx}
\usepackage{subfigure}
\usepackage{booktabs} 

\usepackage{amsfonts}
\usepackage{amsmath}
\usepackage{amssymb}
\usepackage{bm}
\usepackage{caption}
\usepackage{graphicx}
\usepackage{algorithm}
\usepackage{algorithmic}
\usepackage{multirow}
\usepackage{xcolor}
\usepackage{booktabs}
\usepackage{yfonts}
\usepackage{enumitem}
\usepackage{adjustbox}
\usepackage{mathtools}
\DeclarePairedDelimiter\ceil{\lceil}{\rceil}
\DeclarePairedDelimiter\floor{\lfloor}{\rfloor}
\newcommand{\Li}{\phi^{\text{lin}}_t}
\newcommand{\Bx}{\bm{x}}

\newcommand{\E}{\mathbb{E}}

\newcommand{\R}{\ifmmode\mathbb{R}\else$\mathbb{R}$\fi}
\newcommand{\C}{\ifmmode\mathbb{C}\else$\mathbb{C}$\fi}
\newcommand{\N}{\ifmmode\mathbb{N}\else$\mathbb{N}$\fi}
\newcommand{\Q}{\ifmmode\mathbb{Q}\else$\mathbb{Q}$\fi}
\newcommand{\Z}{\ifmmode\mathbb{Z}\else$\mathbb{Z}$\fi}

\newcommand{\wKK}{\widetilde{K}}
\newcommand{\OO}{\mathcal{O}}
\newcommand{\xal}{{\bm{\alpha}}}
\newcommand{\xx}{{\bm{x}}}

\newcommand{\mrm}[1]{\mathrm{#1}}

\def\one{\mbox{1\hspace{-4.25pt}\fontsize{12}{14.4}\selectfont\textrm{1}}} 

\DeclareMathOperator*{\argmin}{arg\,min}

\newtheorem{theorem}{Theorem}
\newtheorem{definition}{Definition}
\newtheorem{lemma}{Lemma}
\newtheorem{proof}{Proof}

\usepackage{hyperref}

\usepackage[accepted]{icml2021}

\icmltitlerunning{Reproducing Activation Function for Deep Learning}

\begin{document}

\twocolumn[
\icmltitle{Reproducing Activation Function for Deep Learning}



\icmlsetsymbol{equal}{*}

\begin{icmlauthorlist}
\icmlauthor{Senwei Liang}{equal,p}
\icmlauthor{Liyao Lyu}{equal,m}
\icmlauthor{Chunmei Wang}{tt}
\icmlauthor{Haizhao Yang}{p}
\end{icmlauthorlist}

\icmlaffiliation{m}{Department of Computational Mathematics, Science, and Engineering, Michigan State University, East Lansing, USA}
\icmlaffiliation{tt}{Department of Mathematics \& Statistics, Texas Tech University, Lubbock, USA}
\icmlaffiliation{p}{Department of Mathematics, Purdue University, West Lafayette, USA}

\icmlcorrespondingauthor{Haizhao Yang}{haizhao@purdue.edu}

\icmlkeywords{Deep Neural Network; Activation Function; Approximation; Neural Tangent Kernel; Partial Differential Equation; Data Reconstruction.}

\vskip 0.3in
]



\printAffiliationsAndNotice{\icmlEqualContribution} 

\begin{abstract}
We propose reproducing activation functions (RAFs) to improve deep learning accuracy for various applications ranging from computer vision to scientific computing. The idea is to employ several basic functions and their learnable linear combination to construct neuron-wise data-driven activation functions for each neuron. Armed with RAFs, neural networks (NNs) can reproduce traditional approximation tools and, therefore, approximate target functions with a smaller number of parameters than traditional NNs. In NN training, RAFs can generate neural tangent kernels (NTKs) with a better condition number than traditional activation functions lessening the spectral bias of deep learning. As demonstrated by extensive numerical tests, the proposed RAFs can facilitate the convergence of deep learning optimization for a solution with higher accuracy than existing deep learning solvers for audio/image/video reconstruction, PDEs, and eigenvalue problems. With RAFs, the errors of audio/video reconstruction, PDEs, and eigenvalue problems are decreased by over 14\%, 73\%, 99\%, respectively, compared with baseline, while the performance of image reconstruction increases by 58\%.
\end{abstract}
\vspace{-0.6cm}
\section{Introduction}
\label{sec:introduction}

Deep neural networks are an important tool for solving a wide range regression problems with surprising performance. 
For example, as a mesh-free representation of objects, scene geometry, and appearance, the so-called ``coordinate-based" networks \cite{tancik2020fourier} take low-dimensional coordinates as inputs and output an object value of the shape, density, and/or color at the given input coordinate. Another example is NN-based solvers for high-dimensional and nonlinear partial differential equations~(PDEs) in complicated domains~\cite{doi:10.1002/cnm.1640100303,han2018solving,RAISSI2019686,Khoo2017SolvingPP}. However, the optimization problem for above applications is highly non-convex making it challenging to obtain a highly accurate solution. Exploring different neural network architectures and training strategies for highly accurate solutions have been an active research direction~\cite{sitzmann2020implicit,tancik2020fourier,JAGTAP2020109136}.
 
We propose RAFs to improve deep learning accuracy. The idea is to employ several basic functions and their learnable linear combination to construct neuron-wise data-driven activation functions. Armed with RAFs, NNs can reproduce traditional approximation tools efficiently, e.g., orthogonal polynomials, Fourier basis functions, wavelets, radial basis functions. Therefore, NNs with the proposed RAF can approximate a wide class of target functions with a smaller number of parameters than traditional NNs. Therefore, the data-driven activation functions are called reproducing activation functions (RAFs). 

In NN training, RAFs can empirically generate NTKs with a better condition number than traditional activation functions lessening the spectrum bias of deep learning. NN optimization usually can only find the smoothest solution with the fastest decay in the frequency domain due to the implicit regularization of network structures~\cite{FP,cao2019towards,Neyshabur2017,Lei2018}, which can be generalized to PDE problems, e.g., the optimization and generalization analysis \cite{LuoYang2020} and the spectral bias \cite{wang2020pinns}. Therefore, designing an efficient algorithm to identify oscillatory or singular solutions to regression and PDE problems is challenging. 

\textbf{Contribution.} We summarize our contribution as follows, 

1. We propose RAFs and their approximation theory. NNs with this activation function can reproduce traditional approximation tools (e.g., polynomials, Fourier basis functions, wavelets, radial basis functions) and approximate a certain class of functions with exponential and dimension-independent approximation rates.

2. Empirically, RAFs can generate NTKs with a smaller condition number than traditional activation functions lessening the spectrum bias of NNs.

3.Extensive experiments on coordinate-based data representation and PDEs demonstrate the effectiveness of the proposed activation function.

\section{Related Works}
\label{sec:related}
\textbf{NN-based PDE solvers. } First of all, NNs as a mesh-free parametrization can efficiently approximate various high-dimensional solutions with dimension-independent approximation rates \cite{yarotsky2019,MO,HJKN19_814,Shen4,Shen5} and/or achieving exponential approximation rates  \cite{DBLP:journals/corr/abs-1807-00297,Opschoor2019,Shen4,Shen5}. Second, NN-based PDE solvers enjoy simple implementation and work well for nonlinear PDEs on complicated domains. There has been extensive research on improving the accuracy of these PDE solvers, e.g., improving the sampling strategy of SGD \cite{NakamuraZimmerer2019AdaptiveDL,Chen2019QuasiMonteCS} or the sample weights in the objective function \cite{Gu2020}, building physics-aware NNs  \cite{Cai2019APS,Liu2020Jul,gu2020structure}, combining traditional iterative solvers \cite{FP,huang2020int}.  

Take the example of boundary value problems (BVP) and the least squares method \cite{doi:10.1002/cnm.1640100303}. Consider the BVP
\begin{equation}\label{eqn:BVP}
\begin{split}
&\mathcal{D}u(\bm{x})=f(u(\bm{x}),\bm{x}),\text{~in~}\Omega,\\
&\mathcal{B}u(\bm{x})=g(\bm{x}),\text{~on~}\partial\Omega,
\end{split}
\end{equation}
where $\mathcal{D}: \Omega\rightarrow \Omega$ is a differential operator that can be nonlinear, $\Omega\subset \mathbb{R}^d$ is a bounded domain, and $\mathcal{B}u=g$ characterizes the boundary condition. Special network structures $\phi(\bm{x};\bm{\theta})$ parameterized by $\bm{\theta}$ can be proposed such that  $\phi(\bm{x};\bm{\theta})$ can satisfy boundary conditions, and the loss function over the given points $\{x_i\}_{i=1}^N$ becomes 
\begin{equation}\label{eqn:dloss2}
    \begin{aligned}
    \min_{\bm{\theta}}\hat{\mathcal{L}} (\bm{\theta}):= \frac{1}{N}\sum_{i=1}^N   \big(\mathcal{D}u(\bm{x}_i; \bm{\theta}) - f(\bm{x}_i)\big)^2 .
    \end{aligned}
\end{equation}

\textbf{Coordinate-based network. } Deep NNs were used as a mesh-free representation of objects, scene geometry, and appearance (e.g. meshes and voxel grids), resulting in notable performance compared to traditional discrete representations. This strategy is compelling in data compression and reconstruction, e.g., see \cite{tancik2020fourier,8953765,Jeruzalski2020NASANA,9157823,9010266,8954065,9156730,Saito2019PIFuPI,Sitzmann2019SceneRN}. 

\textbf{Neural tangent kernel.} NTK is a tool to study the training behavior of deep learning in regression problems and PDE problems~\cite{DBLP:journals/corr/abs-1806-07572,cao2019towards,LuoYang2020,wang2020pinns}. Let $\mathcal{X}$ be $\{\bm{x}_i\}_{i=1}^N$, $\mathcal{Y}$ be the set of function values, $\hat{\mathcal{J}}(\bm{\theta})$ be the square error loss for regression problem. Using gradient flow to analyze the training dynamics of $\hat{\mathcal{J}}(\bm{\theta})$, we have the following evolution equations:
$
\dot{\bm{\theta}}_t = -\nabla_{\bm{\theta}}\phi_t(\mathcal{X})^T\nabla_{\phi_t(\mathcal{X})}\hat{\mathcal{J}},
$
and
$
\dot{\phi}_t(\mathcal{X})=\nabla_{\bm{\theta}}\phi_t(\mathcal{X})\dot{\bm{\theta}}_t=-\hat{\Theta}_t(\mathcal{X},\mathcal{X})\nabla_{\phi_t(\mathcal{X})}\hat{\mathcal{J}},
$
where $\bm{\theta}_t$ is the parameter set at iteration time $t$, $\phi_t(\mathcal{X})=\text{vec}([\phi_t(\bm{x};\bm{\theta}_t)]_{\bm{x}\in\mathcal{X}})$ is the $N\times 1$ vector of concatenated function values for all samples, and $\nabla_{\phi_t(\mathcal{X})}\hat{\mathcal{J}}$ is the gradient of the loss with respect to the network output vector $\phi_t(\mathcal{X})$, $\hat{\Theta}_t:=\hat{\Theta}_t(\mathcal{X},\mathcal{X})$ in $\mathbb{R}^{N\times N}$ is the NTK at iteration time $t$ defined by
\[
\hat{\Theta}_t=\nabla_{\bm{\theta}}\phi_t(\mathcal{X})\nabla_{\bm{\theta}}\phi_t(\mathcal{X})^T.
\]
The NTK can also be defined for general arguments, e.g., $\hat{\Theta}_t(\bm{x},\mathcal{X})$ with $\bm{x}$ as a test sample location. 

If the following linearized network by Taylor expansion is considered,
$
\Li(\bm{x}):=\phi(\bm{x};\bm{\theta}_0) + \nabla_{\bm{\theta}} \phi(\bm{x};\bm{\theta}_0) \bm{\omega}_t,
$
where $\bm{\omega}_t:=\bm{\theta}_t-\bm{\theta}_0$ is the change in the parameters from their initial values. The closed form solutions are
\[
\bm{\omega}_t = -\nabla_{\bm{\theta}}\phi_0(\mathcal{X})^T \hat{\Theta}^{-1}_0 \left(I-e^{-\hat{\Theta}_0 t}\right)(\phi_0(\mathcal{X})-\mathcal{Y}),
\]
and
\begin{equation}\label{eqn:evs}
\Li(\bm{x}) - \phi_0(\bm{x})= \hat{\Theta}_0(\bm{x},\mathcal{X})\hat{\Theta}_0^{-1}\left(I-e^{-\hat{\Theta}_0 t}\right)(\mathcal{Y}-\phi_0(\mathcal{X})).
\end{equation}

There mainly two kinds of observations from \eqref{eqn:evs} from the perspective of kernel methods. The first one is through the eigendecomposition of the initial NTK. If the initial NTK is positive definite, $\Li$ will eventually converge to a neural network that fits all training examples and its generalization capacity is similar to kernel regression by \eqref{eqn:evs}. The error of $\Li$ along the direction of eigenvectors of $\hat{\Theta}_0$ corresponding to large eigenvalues decays much faster than the error along the direction of eigenvectors of small eigenvalues, which is referred to as the spectral bias of deep learning. The second one is through the condition number of the initial NTK. Since NTK is real symmetric, its condition number is equal to its largest eigenvalue over its smallest eigenvalue. If the initial NTK is positive definite, in the ideal case when $t$ goes to infinity, $\left(I-e^{-\hat{\Theta}_0 t}\right)(\phi_0(\mathcal{X})-\mathcal{Y})$ in \eqref{eqn:evs} approaches to $\phi_0(\mathcal{X})-\mathcal{Y}$ and, hence, $\Li(\bm{x})$ goes to the desired function value for $\bm{x}\in\mathcal{X}$. However, in practice, when $\hat{\Theta}_0$ is very ill-conditioned, a small approximation error in $\left(I-e^{-\hat{\Theta}_0 t}\right)(\phi_0(\mathcal{X})-\mathcal{Y})\approx \phi_0(\mathcal{X})-\mathcal{Y}$ may be amplified significantly, resulting in a poor accuracy for $\Li(\bm{x})$ to solve the regression problem.

The above discussion is for the NTK in regression setting. In the case of PDE solvers, we introduce the NTK below
\begin{equation}\label{eqn:NTKPDE}
\hat{\Theta}_t=\left(\nabla_{\bm{\theta}}\mathcal{D}\phi_t(\mathcal{X})\right)\left(\nabla_{\bm{\theta}}\mathcal{D}\phi_t(\mathcal{X})\right)^T,
\end{equation}
where $\mathcal{D}$ is the differential operator of the PDE. Similar to the discussion for regression problems, the spectral bias and the conditioning issue also exist in deep learning based PDE solvers by almost the same arguments.

\label{sec:pre}

\section{Reproducing Activation Functions}\label{sec:RAF}

\subsection{Abstract Framework}\label{sec:AF}

The concept of RAFs is to apply different activation functions in different neurons. Let $\mathcal{A}=\{\gamma_1(x),\dots,\gamma_P(x)\}$ be a set of $P$ basic activation functions. In the $i$-th neuron of the $\ell$-th layer, an activation function 
\begin{equation}\label{eqn:rpf}
\sigma_{i,\ell}(x)=\sum_{p=1}^P \alpha_{p,i,\ell} \gamma_p(\beta_{p,i,\ell} x)
\end{equation}
is applied, where $\alpha_{p,i,\ell}$ is a learnable combination coefficient and $\beta_{p,i,\ell}$ is a learnable scaling parameter. Let $\bm{\alpha}$ and $\bm{\beta}$ be the union of all learnable combination coefficients and scaling parameters, respectively, we use $\phi(\bm{x};\bm{\theta},\bm{\alpha},\bm{\beta})$ to denote an NN with $\bm{\theta}$ as the set of all other parameters.

In regression problems, given samples $\{\mathbf{x}_i,y_i\}_{i=1}^N$, the empirical loss with RAFs is
\begin{align}\label{eqn:pop2}
\hat{\mathcal{J}}(\bm{\theta},\bm{\alpha},\bm{\beta})=\frac{1}{2N}\sum_{i=1}^N \left| \phi(\mathbf{x}_i;\bm{\theta},\bm{\alpha},\bm{\beta})-y_i\right|^2.
\end{align}

\subsection{Examples and Reproducing Properties}\label{sec:RP}

\subsubsection{Example $1$: Sine-ReLU}
Sine-ReLU networks proposed in \cite{yarotsky2019} apply sine $\sin(x)$ or ReLU $\max\{0,x\}$ in each neuron. Instead, the proposed RAF here has a set of trainable parameters $\bm{\alpha}$ and $\bm{\beta}$. In fact, $\sin(x)$ can be replaced by any Lipschitz periodic function. Let $F_{r,d}$ be the unit ball of the $d$-dimensional Sobolev space $H^{r,\infty}([0,1]^d)$. We have the following theorem according to Thm.~$6.1$ of \cite{yarotsky2019}.

\begin{theorem}\label{thm:sinrelu} \textnormal{\textbf{(Dimension-Independent and Exponential Approximation Rate)}} Fix $r,d$. Let $\sigma$ be a Lipschitz periodic function with period $T$. Suppose $\sigma(x)>0$ for $x\in (0,T/2)$, $\sigma(x)<0$ for $x\in(T/2,T)$, and $\max_{x\in\mathbb{R}}\sigma(x)=-\min _{x\in\mathbb{R}}\sigma(x)$. For any sufficiently large integer $W>0$ and any $f(\bm{x})\in F_{r,d}$, there exists an NN $\phi(\bm{x};\bm{\theta},\bm{\alpha},\bm{\beta})$ such that: 1) The total number of parameters in $\{\bm{\theta},\bm{\alpha},\bm{\beta}\}$ is less than or equal to $W$; 2) $\phi(\bm{x};\bm{\theta},\bm{\alpha},\bm{\beta})$ is built with RAFs associated with $\mathcal{A}=\{\sigma(x),\max\{0,x\}\}$; 3) $\|f(\bm{x})-\phi(\bm{x};\bm{\theta},\bm{\alpha},\bm{\beta})\|_\infty \leq \text{exp}\left( -c_{r,d} W^{1/2}\right)$ 
    with a constant $c_{r,d}>0$ only depending on $r$ and $d$.
\end{theorem}


There are other types of network structures utilizing both $\sin(x)$ and ReLU activation functions but for different application purposes and with different strategies, e.g.,  \cite{zhong2020reconstructing,mildenhall2020nerf,HAN2020,Liu2020Jul,Wang2020,tancik2020fourier}. 

\subsubsection{Example $2$: Floor-Exponential-Sign}
Recently, networks with super approximation power (e.g., an exponential approximation rate without the curse of dimensionality for H{\"o}lder continuous functions) have been proposed in \cite{Shen4,Shen5}, e.g., the Floor-Exponential-Sign Network that uses one of the following three activation functions in each neuron: 
\begin{equation*}
	\sigma_1(x):=\lfloor x\rfloor, \sigma_2(x):= 2^x, \sigma_3:=\mathcal{T}(x-\lfloor x\rfloor-\frac{1}{2}).
\end{equation*}
Here, 
$
\mathcal{T}(x):= \one_{x\geq 0} =\left\{\begin{matrix}
	1,\ x\ge 0,\\
	0,\ x<0.
\end{matrix}\right.
$
The proposed RAF has a set of trainable parameters $\bm{\alpha}$ and $\bm{\beta}$. By Thm.~$1.1$ in \cite{Shen5}, we have the following theorem.

\begin{theorem}\label{thm:fles} \textnormal{\textbf{(Dimension-Independent and Exponential Approximation Rate)}} Given $f$ in $C([0,1]^d)$ and $W\in \mathbb{N}^+$, there exists an NN $\phi(\bm{x};\bm{\theta},\bm{\alpha},\bm{\beta})$ of width $W$ and depth $4$ built with RAFs associated with $\mathcal{A}=\{\sigma_1(x),\sigma_2(x),\sigma_3(x)\}$ such that, for any $\bm{x}\in [0,1)^d$,
	\begin{equation*}
		|\phi(\bm{x};\bm{\theta},\bm{\alpha},\bm{\beta})-f(\bm{x})|\le 2\omega_f(\sqrt{d})2^{-W}+\omega_{f}(\sqrt{d}\,2^{-W})
	\end{equation*}
	 with at most $2W^2+(d+22)W+1$ parameters.
\end{theorem}

Here, $\omega_f(\cdot)$ is the modulus of continuity of $f$ defined as 
\begin{equation*}
	\omega_f(r)= \sup_{ \bm{x},\bm{y}\in [0,1]^d}\big\{|f(\bm{x})-f(\bm{y})|: \|\bm{x}-\bm{y}\|_2\le r\big\}
\end{equation*}
for any $r\ge0$, where $\|\bm{x}\|_2$ is the length of $\bm{x}\in \mathbb{R}^d$. 

\subsubsection{Example $3$: Poly-Sine-Gaussian}

Finally, we propose the poly-sine-Gaussian network using $\mathcal{A}=\{x,x^2,\sin(x), e^{-x^2}\}$ such that NNs can reproduce traditional approximation tools efficiently, e.g., orthogonal polynomials, Fourier basis functions, wavelets, radial basis functions, etc. Therefore, this new NN may approximate a wide class of target functions with a smaller number of parameters than existing NNs, e.g., ReLU NNs, since existing approximation theory with a continuous weight selection of ReLU NNs are established by using ReLU networks to approximate $x$ and $x^2$ as basic building blocks. We present several theorems proved in Appendix to illustrate the approximation capacity of this new network.

\begin{theorem}[Reproducing Polynomials]\label{thm:NNex2} 
 Assume $P(\xx)= \sum_{j=1}^J c_j \xx^{\xal_j}$ for $\xal_j\in \N^d$. For any $N,L,a,b\in \N^+$ such that $ab\geq J$ and $(L-2b-b\log_2 N)N\geq  b \max_j |\xal_j|$, there exists a poly-sine-Gaussian network $\phi$ with width $2Na+d+1$ and depth $L$ such that
    $
 \phi(\xx)=P(\xx)\quad \text{for any $\xx\in \R^d$.}
    $
\end{theorem}

Thm.~\ref{thm:NNex2} characterizes how well poly-sine-Gaussian networks reproduce arbitrary polynomials including orthogonal polynomials. Compared to the results of ReLU NNs for polynomials in \cite{yarotsky2017,Shen3}, poly-sine-Gaussian networks require less parameters. Orthogonal polynomials are important tools for classical approximation theory and numerical computation. For example, the Chebyshev series lies at the heart of approximation theory.
In particular, for analytic functions, the truncated Chebyshev series defined as $f_n(x)=\sum_{k=0}^{n}c_kT_k(x/M)$ are \textit{exponentially accurate} approximations Thm.~8.2 \cite{trefethen2013}, where $T_k$ is the Chebyshev polynomial of degree $k$ defined on $[-1,1]$. More precisely, for some scalars $M\geq1$ and $s>1$, if we define
$
a_s^M=M\frac{s+s^{-1}}{2}, b_s^M = M\frac{s-s^{-1}}{2},
$
and the \textit{Bernstein $s$-ellipse scaled to $[-M,M]$},
\begin{align*}
E_s^M = \left\{x+iy\in\C\,:\,\frac{x^2}{(a_s^M)^2}+\frac{y^2}{(b_s^M)^2}=1\right\},
\end{align*}
then we have the following theorem.  

\begin{theorem}[Exponential Approximation Rate]\label{thm:analytic}
For any $M\geq1$, $s>1$, $C_f>0$, $0<\epsilon<1$, and any {real-valued} analytic function $f$ on $[-M,M]$ that is analytically continuable to the open ellipse $E_s^M$, where it satisfies $\vert f(x)\vert\leq C_f$, there is a poly-sine-Gaussian network $\phi$ with width $2N+2$ and depth $L$ such that $\left\Vert\phi(x) - f(x)\right\Vert_{L^\infty([-M,M])}\leq\epsilon$, where $N$ and $L$ are positive integers satisfying $(L-2n-2-(n+1)\log_2 N)N\geq n(n+1)$ and $n=\OO\left(\frac{1}{\log_2s}\log_2\frac{2C_f}{\epsilon}\right)$.
\end{theorem} 

By choosing $N=\OO(n)$ and $ L=\OO(n\log_2(n))$ in Thm.~\ref{thm:analytic}, the width and depth of $\phi$ are  $\OO\left(\log_2\frac{1}{\epsilon}\right)$ and $\OO\left(\left(\log_2\frac{1}{\epsilon}\right)\log_2\left(\log_2\frac{1}{\epsilon}\right)\right)$, respectively, leading to a network size smaller than that of the ReLU NN in Thm.~$2.6$ in \cite{bandlimit}.

Next, we prove the approximation of poly-sine-Gaussian networks to generalized bandlimited functions below.

\begin{definition} {Let $d\geq 2$ be an integer, $M\geq 1$ be a scalar, and $B=[0,1]^d$. Suppose $K:\R\rightarrow\C$ is analytic and bounded by a constant $D_K\in(0,1]$ on $[-dM,dM]$ and $K$ satisfies the assumption of Thm.~$\ref{thm:analytic}$ for $s>1$ and $C_K>0$. We define the Hilbert space $\mathcal{H}_{K,M}(B)$ of generalized bandlimited functions via}
\begin{align*}
\mathcal{H}_{K,M}(B)&=\bigg\{ f(\bm{x})=\int_{[-M,M]^d}F(\bm{w})K(\bm{w}\cdot\bm{x})d\bm{w}\; \\&
\bigg\arrowvert \; F:[-M,M]^d\rightarrow\C\text{ is in }L^2([-M,M]^d) \bigg\},
\end{align*}
with $\langle f, g\rangle_{\mathcal{H}_{K,M}(B)}:=\int_{[-M,M]^d} F_f(\bm{w}) \overline{F}_g(\bm{w}) d\bm{w}$ and its induced norm $\|f\|_{\mathcal{H}_{K,M}(B)}$, where $F_f=\arg\min_{F\in S_f} \|F\|_{L^2([-M,M]^d)}$ and $S_f=\bigg\{F\;\bigg\arrowvert\; f(\bm{x})=\int_{[-M,M]^d}F(\bm{w})K(\bm{w}\cdot\bm{x})d\bm{w}\bigg\}$.

\end{definition}
Note that $\mathcal{H}_{K,M}(B)$ is a reproducing kernel Hilbert space (RKHS); a classical example of interest is $K(t)=e^{it}$. For simplicity, we will use $F$ instead of $F_f$ for $f\in\mathcal{H}_{K,M}(B)$, when the dependency on $f$ is clear.


\begin{theorem}[Dimension-Independent Approximation]\label{thm:bandlimited}
For any real-valued function $f$ in $\mathcal{H}_{K,M}(B)$, $M\geq 1$, $s>1$, $C_K>0$, and $d\geq 2$. Let us assume that $\int_{\R^d}\vert F(\bm{w})\vert d\bm{w} = \int_{[-M,M]^d}\vert F(\bm{w})\vert d\bm{w} = C_F$. For any measure $\mu$ and $\epsilon\in(0,1)$, there exists a poly-sine-Gaussian network $\phi$ on $B=[0,1]^d$, that has width $\OO\left(\frac{4C_F\sqrt{\mu(B)}}{\epsilon^2\log_2s}\log_2\frac{4C_F\sqrt{\mu(B)}C_K}{\epsilon}\right)$ and depth $\resizebox{.9\hsize}{!}{$\OO\left(\left(\frac{1}{\log_2s}\log_2\frac{4C_F\sqrt{\mu(B)}C_K}{\epsilon}\right) \log_2 \log_2\frac{4C_F\sqrt{\mu(B)}C_K}{\epsilon}\right)$}$ such that
$\left\Vert\phi - f\right\Vert_{L^2(\mu, B)} = \sqrt{\int_B\vert\phi(\bm{x}) - f(\bm{x})\vert^2 d\mu(\bm{x})}\leq\epsilon$.
\end{theorem}


Poly-sine-Gaussian networks can also reproduce typical applied harmonic analysis tools as in the following lemma.

\begin{lemma}\label{lem:NNex3}
\begin{enumerate}[label=(\roman*)]
\item Poly-sine-Gaussian networks can reproduce all basis functions in the discrete cosine transform and the discrete windowed cosine transform with a Gaussian window function in an arbitrary dimension.
\item Poly-sine-Gaussian networks with complex parameters can reproduce all basis functions in the discrete Fourier transform and the discrete Gabor wavelet transform in an arbitrary dimension.
\end{enumerate}
\end{lemma}

\begin{figure*}[ht]
\centering
\begin{adjustbox}{width=0.9\textwidth, center}
	\includegraphics{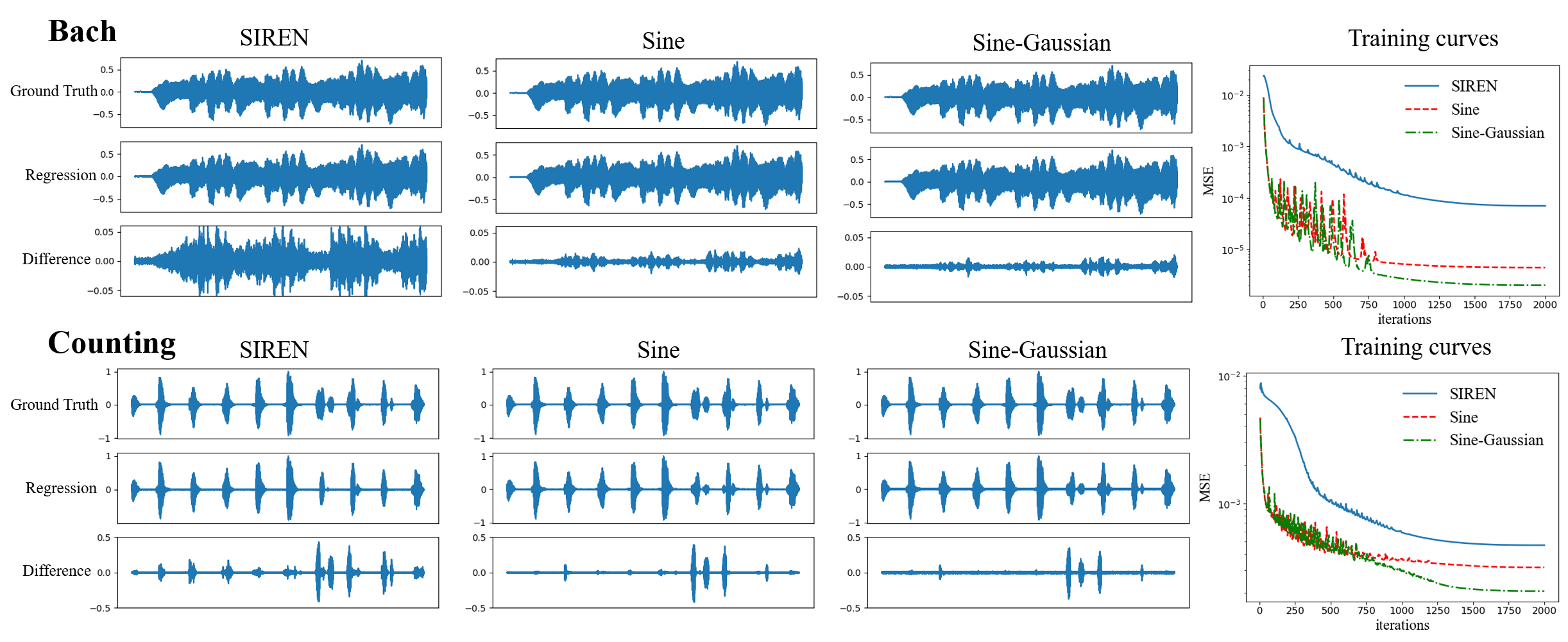}
\end{adjustbox}
\vspace{-0.8cm}
\caption{The comparison of fitted signal and training curves on Bach and Counting for SIREN and our RAF (Sine and Sine-Gaussian). }
\label{fig:audio}
\vspace{-0.4cm}
\end{figure*}

Lem.~\ref{lem:NNex3} above implies that poly-sine-Gaussian networks may be useful in many computer vision and audio tasks involving Fourier transforms and wavelet transforms. Due to the advantage of wavelets to represent functions with singularity, poly-sine-Gaussian networks may also be useful in representing functions with singularity. We would like to highlight that the Gaussian function may not be the optimal choice in the concept of RAF. Other window functions in wavelet analysis may provide better performance and this would be problem-dependent. 

Finally, we have the next lemma for radial basis functions. 

\begin{lemma}\label{lem:NNex4} 
Poly-sine-Gaussian networks can reproduce Gaussian radial basis functions and approximate radial basis functions defined on a bounded closed domain with analytic kernels with an exponential approximation rate.
\end{lemma}

We will end this section with an informal discussion about the NTK of poly-sine-Gaussian networks. As we shall discuss in Section \ref{sec:related}, deep learning can be approximated by kernel methods with a kernel $\hat{\Theta}_0$ in \eqref{eqn:evs}. Therefore, from the perspective of kernel regression for regressing $f(\Bx)$ with training samples $\{(\Bx_i,f(\Bx_i))\}_{i=1}^N$, $\hat{\Theta}_0(\Bx,\Bx_i)$ quantifies the similarity of the point $\Bx$ and a training point $\Bx_i\in \mathcal{X}$ and, hence, serves as a weight of $f(\Bx_i)$ in a toy regression formulation: $\phi(\Bx;\bm{\omega}):=\sum_{i=1}^N \omega_i f(\Bx_i)\hat{\Theta}(\Bx,\Bx_i)$, where $\bm{\omega}=[\omega_1,\dots,\omega_N]$ is a set of learnable parameters and $\phi(\Bx;\bm{\omega})$ is the approximant of the target function $f(\Bx)$. To enable a kernel method to learn both smooth functions and highly oscillatory functions, the kernel function $\hat{\Theta}$ should have a widely spreading Fourier spectrum. By using $\sin(\beta x)$ with a tunable $\beta$ in the poly-sine-Gaussian, the poly-sine-Gaussian network could learn an appropriate kernel for both kinds of functions. Similarly, by using $\text{exp}(-(\beta x)^2))$ with a tunable $\beta$ in the poly-sine-Gaussian, the poly-sine-Gaussian network could learn an appropriate kernel for both smooth and singular functions. We will provide numerical examples to demonstrate this empirically in the next section.

\begin{table*}
  \begin{minipage}[t]{0.6\textwidth}
  \centering
  \captionof{table}{The comparison of PSNR/SSIM of the fitted images using different activation functions. The larger these numbers are, the better the performance is.}
  \scalebox{0.85}{
\begin{tabular}{ccccc}
    \toprule
    Activation & Camera & Astronaut & Cat   & Coin \\
    \midrule
    SIREN & 45.80/0.9913 & 44.84/0.9962 & 49.58/0.9970 & 43.05/0.9868 \\
    Sine  & 60.60/0.9995 & 59.37/0.9997 & 65.94/0.9999 & 62.66/0.9998 \\
    Poly-Sine & 61.21/0.9996 & 59.99/0.9997 & 66.41/0.9999 & 63.57/0.9998 \\
    Poly-Sine-Gauss. & 73.80/1.0000 & 70.98/1.0000 & 82.55/1.0000 & 74.92/1.0000 \\
    \bottomrule
    \end{tabular}%
}
\label{tab:image}
    \end{minipage}
  \hspace{0.4cm}
  \begin{minipage}[t]{0.35\textwidth}
  \captionof{table}{The comparison of PSNR of videos fitted by different activation functions. The mean and average are computed over $250$ frames.}
  \scalebox{0.9}{
\begin{tabular}{ccc}
    \toprule
    Activation & Mean PSNR & Std PSNR \\
    \midrule
    SIREN & 32.17  & 2.16  \\
    Sine-Gaussian  & 32.79  & 2.10  \\
    \bottomrule
    \end{tabular}%
}
\label{tab:video}
\end{minipage}
\end{table*}


\section{Numerical Results}\label{sec:result}
In this section, we will illustrate the advantages of RAFs in two kinds of applications, data representation and scientific computing. The optimal choice of basic activation functions would be problem-dependent. 

\subsection{Coordinate-based Data Representation}
We verify the performance of RAFs on data representations using coordinate-based NNs. Mean square error~(MSE) quantifies the difference between the ground truth and the NN output. Standard NNs, e.g., ReLU NNs, were shown to have poor performance to fit high-frequency components of signals~\cite{sitzmann2020implicit,tancik2020fourier}. SIREN activation function~\cite{sitzmann2020implicit}, i.e., $\sin(30x)$, improves the ability of NNs to represent complex signals. 

As we discussed in Section \ref{sec:RP}, the SIREN function is a special case of the poly-sine-Gaussian activation function in our framework. We will show that poly-sine-Gaussian activation function can provide better performance than SIREN when the combination coefficients $\bm{\alpha}$ and the scaling parameters $\bm{\beta}$ are specified or trained appropriately in a problem-dependent manner. We follow the official implementation of SIREN on representations of audio, image, and video signal~(refer to~\cite{sitzmann2020implicit} for details). The main difference between the SIREN code and ours is the activation function. All trainable parameters are trained to minimize the empirical loss function in \eqref{eqn:pop2}. The NN is optimized by Adam optimizer with an initial learning rate $10^{-4}$ and cosine learning rate decay. 

\subsubsection{Audio Signal} 
\label{sec:audio}
We start from modeling audio signals on two audio clips, Bach and Counting as shown in Figure~\ref{fig:audio}. An NN is trained to regress from a one-dimensional time coordinate to the corresponding sound level. Note that audio signals are purely oscillatory signals. Therefore, in the reproducing activation framework, $x$ and $x^2$ are not necessary. We apply two forms of RAFs, Sine and Sine-Gaussian. The Sine one is set as  
$
    \alpha_1  \sin(\beta_1x),
$
while the Sine-Gaussian one is set as 
$
    \alpha_1  \sin(\beta_1x)+\alpha_2\text{exp}(-x^2/(2\beta_2^2)),
$
where $\alpha_1$ is initialized as $\mathcal{N}(2,0.1)$, $\alpha_2$ is initialized as $\mathcal{N}(1.0,0.1)$, $\beta_1$ is initialized as $\mathcal{N}(30,0.001)$, and $\beta_2$ is initialized with a uniform distribution $\mathcal{U}(0.01,0.05)$. We use a 3-hidden-layer neural network with $256$ neurons per layer to fit the audio signal following the network structure of SIREN. The NNs are trained for 2000 iterations. Figure~\ref{fig:audio} displays the fitted signals and training curves.  Figure~\ref{fig:audio} shows our method has the capacity of modeling the audio signals more accurately than SIREN and leads to a smaller error in regression. Besides, our RAFs can converge to a better local minimum at a faster speed compared with SIREN. Moreover, we can see the add-in Gaussian function enhances the fitting ability. 

\begin{figure}[ht]
\centering
\begin{adjustbox}{width=0.25\textwidth, center}
	\includegraphics{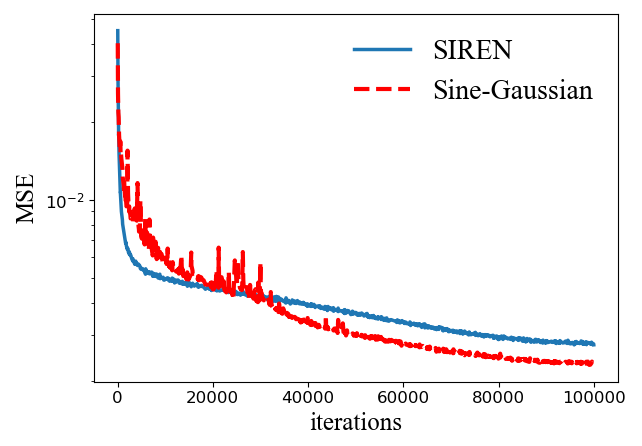}
\end{adjustbox}
\vspace{-0.8cm}
\caption{Comparison of training curves on video fitting for different activation functions.}
\label{fig:bikevideo}
\end{figure}


\begin{table*}[ht]
  \centering
  \caption{The relative $L^2$ error of different activation functions for the regression problem, Poisson equation \eqref{eqn:regular}, PDE with low regularity \eqref{eqn:singular}, PDE with an oscillatory solution \eqref{eqn:oscillation} and Eigenvalue problem with $d=5$ or $d=10$. $\oplus$ means the concatenation of different activation functions in the network.}
  \begin{adjustbox}{width=0.9\textwidth, center}
      \begin{tabular}{ccccccc}
    \toprule
    Examples & Regression & Poisson Equation & Low Regularity & Oscillation & Eigen. $(d=5)$ & Eigen. $(d=10)$\\
    \midrule
    $\text{ReLU}$ & 6.61 e-02 & - & - & - & 6.38e-03 & 4.59 e-03\\
    $\text{ReLU}^3$ & 1.13 e-01 & 1.38 e-03 & 1.49 e-03 & 3.16 e-05 & 0.307 & 0.223\\
    $x \oplus x^2$  & 3.71 e-01 & 4.48 e-04 & 4.39 e-02 & 9.46 e-02 & -&-\\
    $x \oplus x^2 \oplus \text{ReLU}$ & 9.98 e-02 & 4.48 e-04 & 6.06 e-01 & 3.81 e+00 & -&-\\
   	$x \oplus x^2 \oplus \text{ReLU}^3$ & 9.12 e-02  & 1.40 e-03 & 9.48 e-04 & 3.15 e-05 &- &-\\
    $x \oplus x^2 \oplus \sin(x)$ & 9.07 e-02 & 4.18 e-04 & 2.42 e-03 & 4.69 e-06 & - & -\\
    $x \oplus x^2 \oplus \sin(x) \oplus$ Gaussian & 3.46 e-02 & 6.87 e-05 & 1.91 e-04 & 3.35 e-06 & 2.09 e-03 & 1.10 e-03 \\
    Rational \cite{boulle2020rational}& 3.94 e-02 & -     & -     & - & - & - \\
    \bottomrule
    \end{tabular}%
  \end{adjustbox}
  \label{tab:pdes}%
  \vspace{-0.5cm}
\end{table*}%
\begin{figure*}[htbp]
\centering     
\subfigure[Regression]{\label{fig:regressions}\includegraphics[width=.2\textwidth]{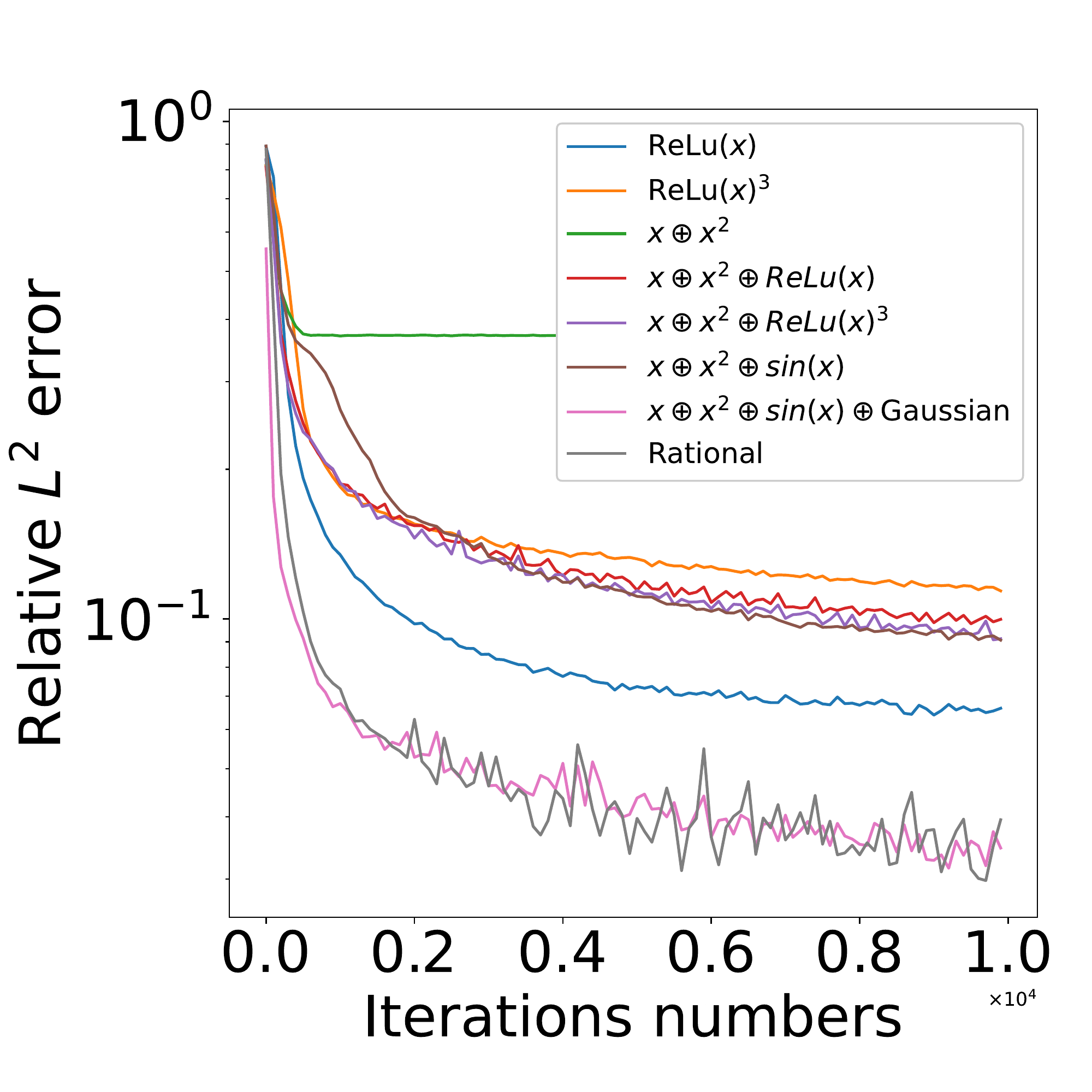}}\vspace{-0.35cm}
\subfigure[Poisson Equation]{\label{fig:poisson}\includegraphics[width=.2\textwidth]{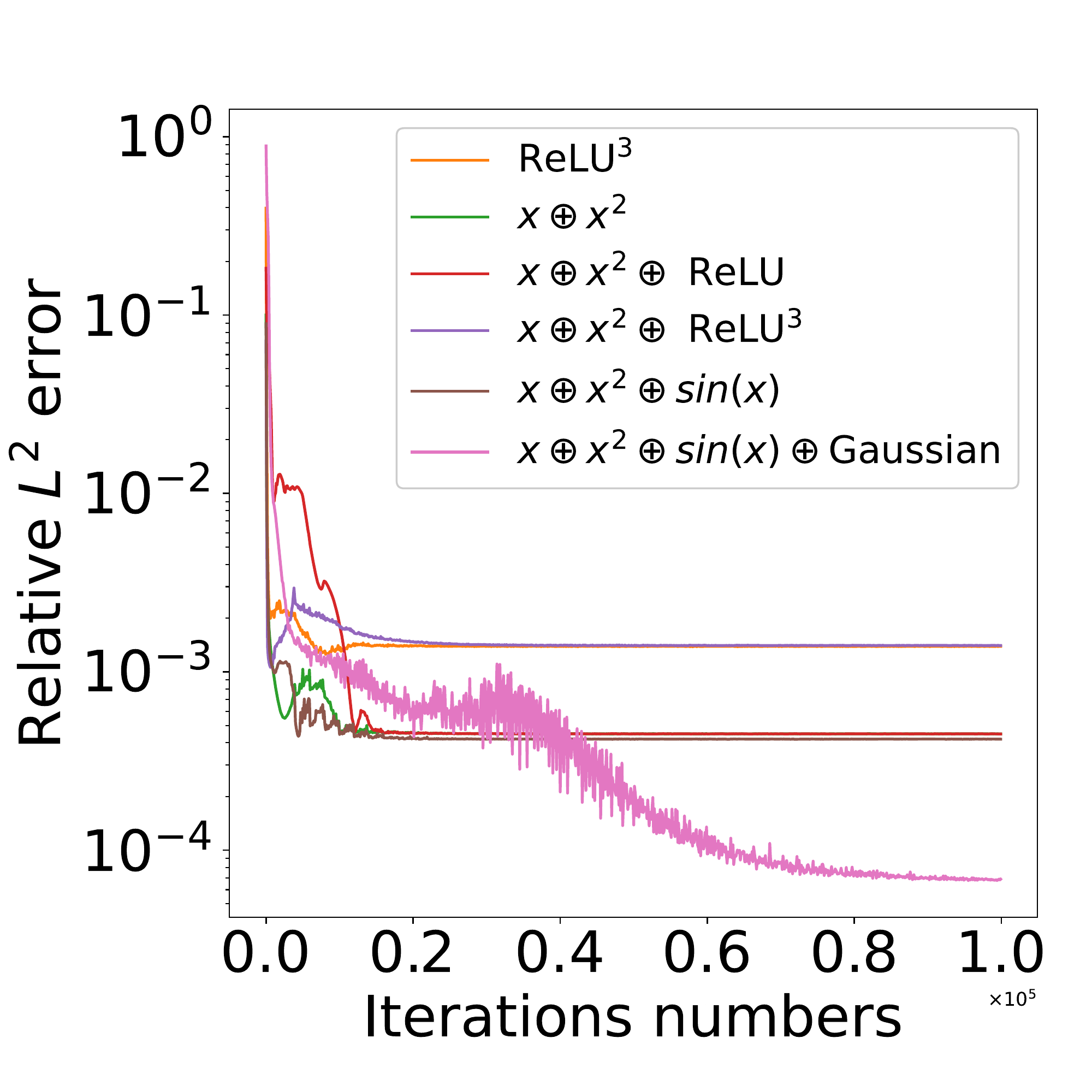}}
\subfigure[Low Regularity]{\label{fig:lowregularity}\includegraphics[width=.2\textwidth]{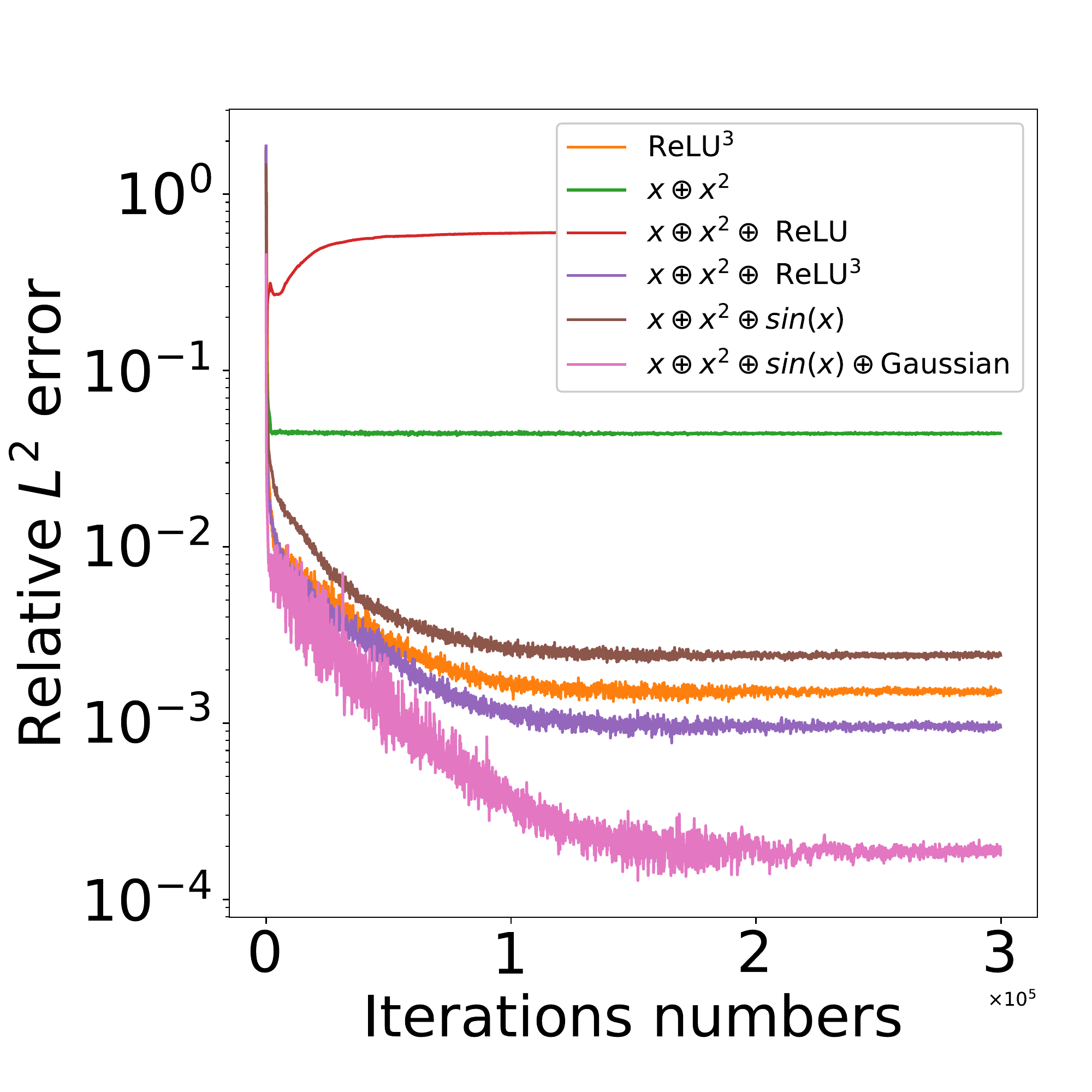}}
\subfigure[Oscillation]{\label{fig:oscillation}\includegraphics[width=.2\textwidth]{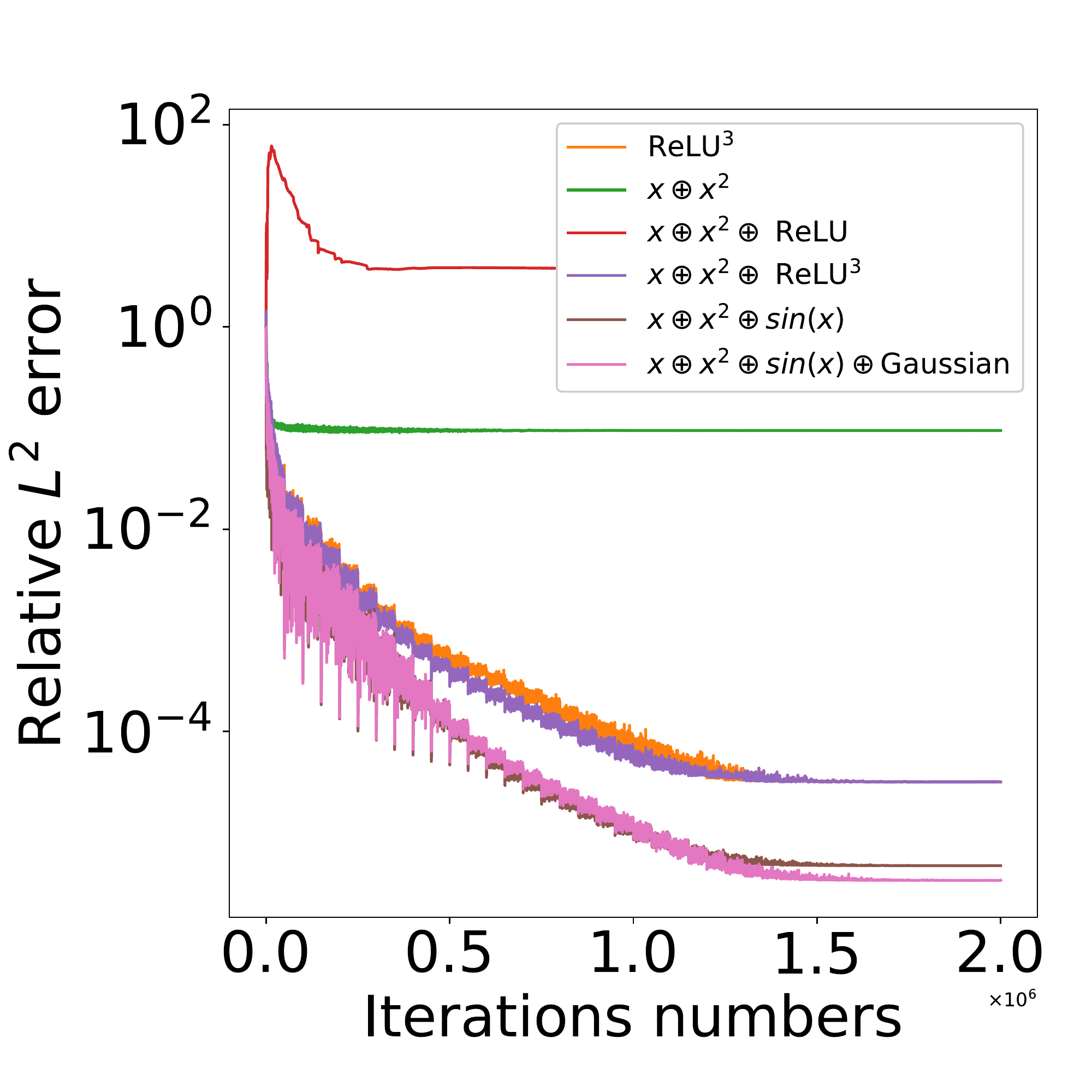}}
\subfigure[Oscillation]{\label{fig:oscillation2}\includegraphics[width=.2\textwidth]{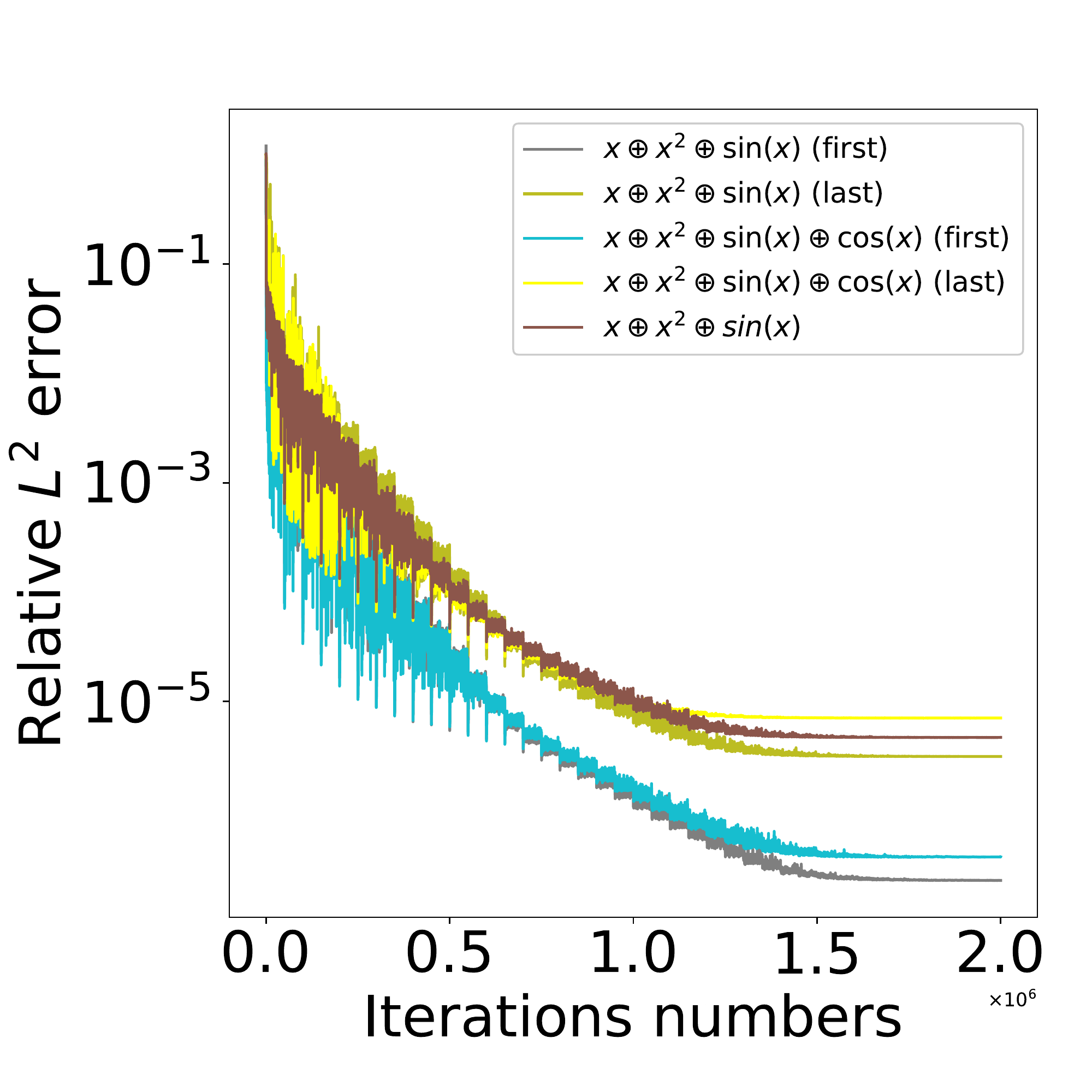}}
\subfigure[Oscillation]{\label{fig:oscillation20}\includegraphics[width=.2\textwidth]{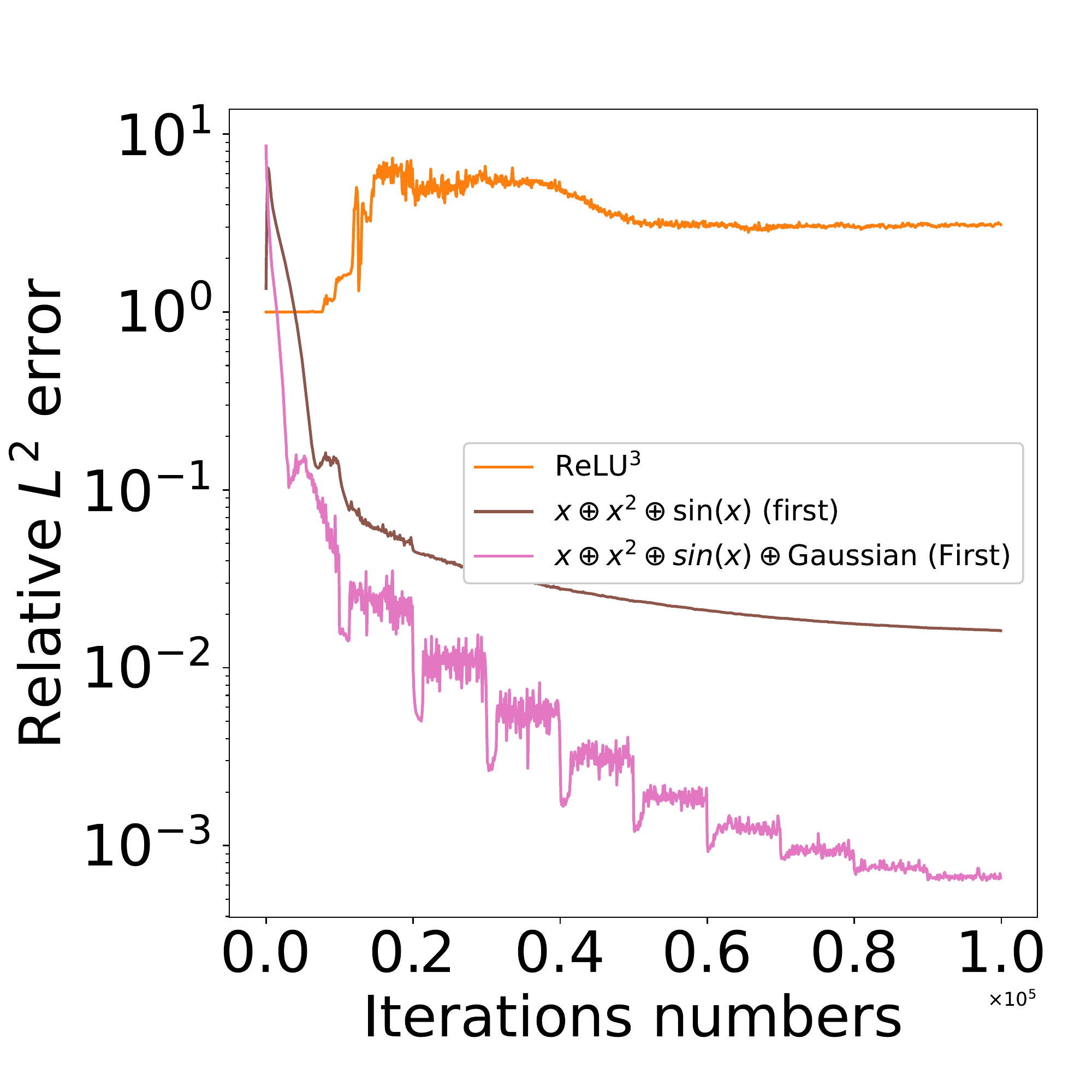}}
\subfigure[Eigenvalue (d=5)]{\label{fig:eigen5}\includegraphics[width=.2\textwidth]{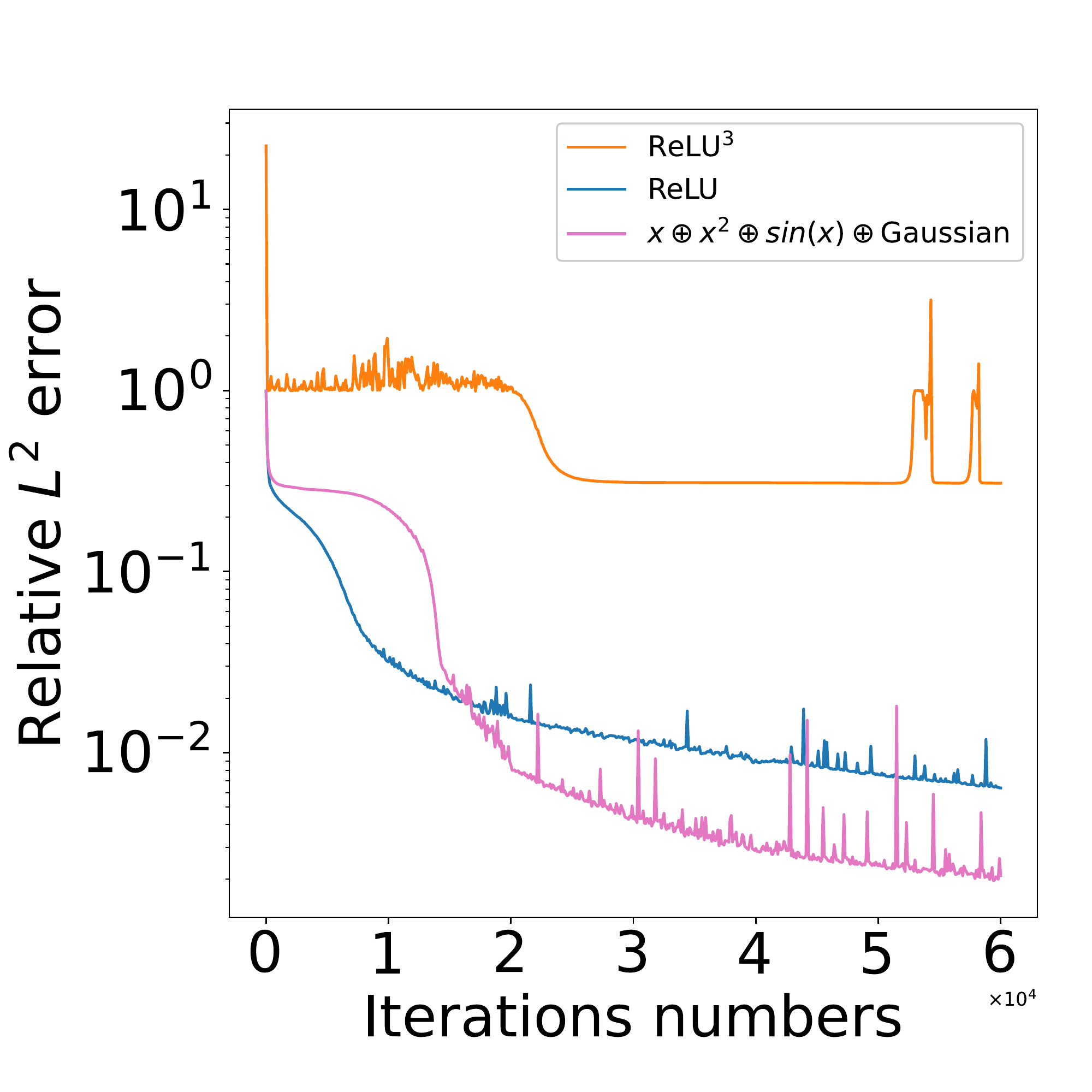}}
\subfigure[Eigenvalue (d=10)]{\label{fig:eigen10}\includegraphics[width=.2\textwidth]{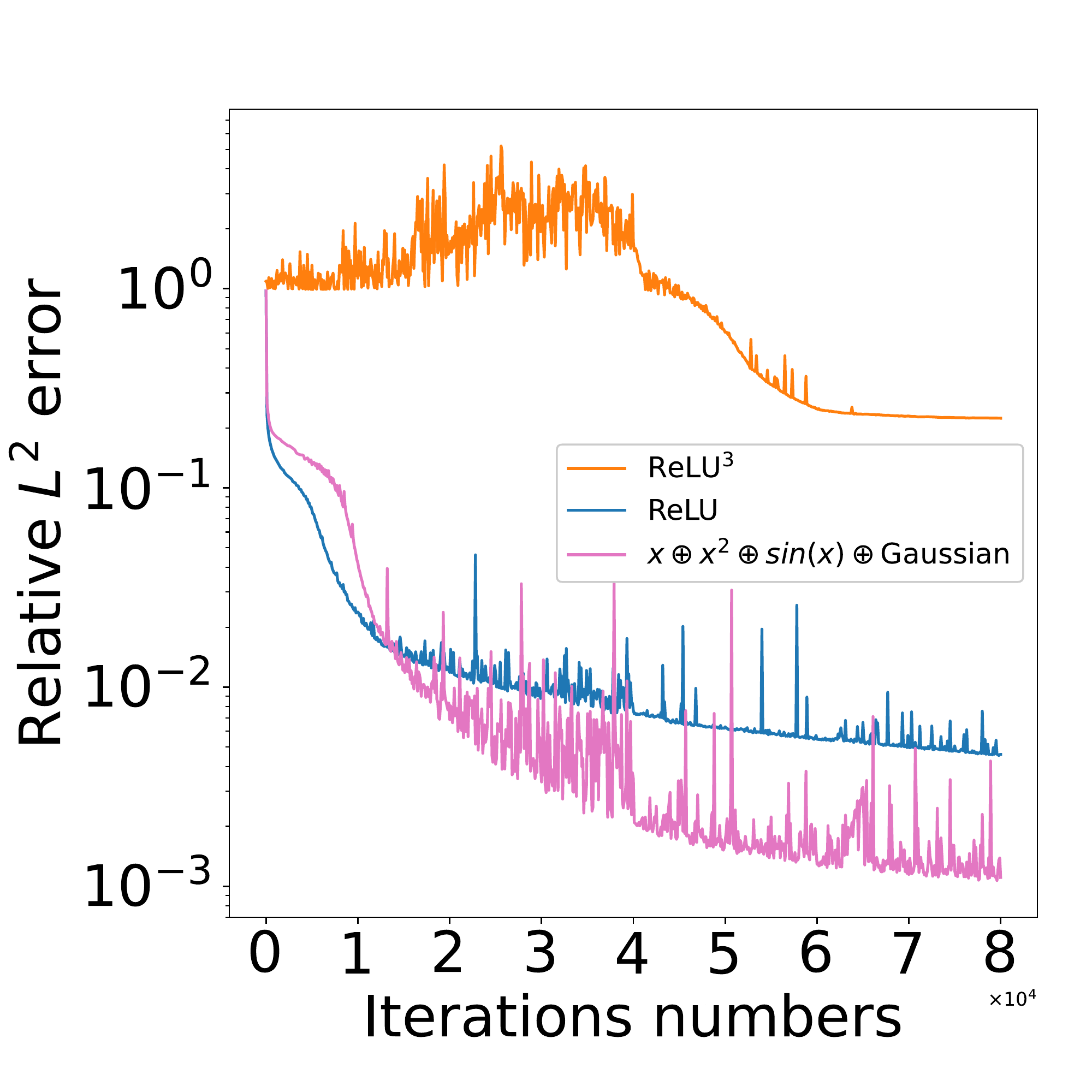}}
\vspace{-0.4cm}
\caption{The relative $L^2$ error vs. iteration of different activation functions for (a) the regression, (b) Poisson equation \eqref{eqn:regular}, (c) PDE with low regularity \eqref{eqn:singular}, (d) PDE with an oscillatory solution \eqref{eqn:oscillation}, PDE with (e) an oscillatory solution or (f) a super oscillatory solution solved by using different scaling parameters in trigonometric activation function, and (g-h) Eigenvalue problem with $d=5$ or $d=10$.}
\vspace{-0.4cm}
\end{figure*}

\subsubsection{Image Signal} We regress a grayscale image by learning a mapping from two-dimensional pixel coordinates to the corresponding pixel value. Four image of size 256$\times$256 are used, including Camera, Astronaut, Cat and Coin, which are available in Python Pillow. Note that images usually contain a cartoon part and a texture part. We apply three types of RAFs for image fitting: Sine, $\alpha_1  \sin(\beta_1x)$; Poly-Sine, $\alpha_1 \sin(\beta_1 x)+ \alpha_3 x + \alpha_4 x^2
    \label{eqn:imagefitting_sine_poly}$;
and Poly-Sine-Gaussian,
\begin{align}
    \alpha_1 \sin(\beta_1 x)+ \alpha_2\text{exp}(-x^2/(2\beta_2^2))+\alpha_3 x + \alpha_4 x^2.
    \label{eqn:imagefitting_sine_poly_gaussian}
\end{align}

Here, $\alpha_1$, $\alpha_2$, $\alpha_3$, $\alpha_4$, $\beta_1$, and $\beta_2$ are initialized as $\mathcal{N}(2,0.1)$, $\mathcal{N}(1,0.1)$, $\mathcal{N}(0.0,0.1)$, $\mathcal{N}(1.0,0.1)$, $\mathcal{N}(30,0.001)$, and $\mathcal{U}(0.01,0.05)$, respectively. An NN with $3$ hidden layers and $256$ neurons per layer is trained for $2,000$ iterations. Table~\ref{tab:image} summarizes the Peak signal-to-noise ratio (PSNR) and Structural similarity~(SSIM) of the fitted images showing that RAFs outperform SIREN with a significant margin. 


 

\subsubsection{Video Signal}
We fit a color video named Bike with $250$ frames available in Python Skvideo Package. The regression is from three-dimensional coordinates to RGB pixel values. 
We apply Sine-Guassian as defined in Section \ref{sec:audio}, but $\alpha_1$, $\alpha_2$, $\beta_1$ and $\beta_2$ are initialized by $\mathcal{N}(1,0.1)$, $\mathcal{N}(1,0.1)$, $\mathcal{N}(30,0.001)$ and $\mathcal{U}(0.002,0.01)$, respectively. An NN with $3$ hidden layers and $400$ neurons per layer is trained for $100,000$ iterations. Figure~\ref{fig:bikevideo} displays 
the training curves of video fitting for different activation function. Table~\ref{tab:video} shows the mean and standard derivation of PSNR for video over $250$ frames.
From Figure~\ref{fig:bikevideo}, 
RAFs can lead to a better minimizer with a larger PSNR than SIREN.

\subsection{Scientific Computing Applications}

We compare RAFs with popular activation functions in scientific computing and provide ablation study to justify the combination of $\mathcal{A}=\{x,x^2,\sin(x),\text{exp}(-x^2)\}$. The relative $L^2$  error is defined by
$
    \left( \frac{\sum_{i=1}^N \left(u(x_i)-\hat{u}(x_i)\right)^2}{\sum_{i=1}^N u^2(x_i)}\right)^{\frac{1}{2}},
$
where $\{x_i\}_{i=1}^N$ are random points uniformly sampled in the domain, $u$ is the true solution, and $\hat{u}$ is the estimated solution. We will adopt two metrics to quantify the performance of activation functions. The first one is the relative $L^2$ error on test samples. The second metric is the condition number of the NTK matrices for PDE solvers. A smaller condition number usually leads to a smaller iteration number to achieve the same accuracy.

\textbf{Network setting.} 
In all examples, we employ ResNet with two residual blocks and each block contains two hidden layers. Unless specified particularly, the width is set as $50$ and all weights and biases in the $\ell$-th layer are initialized by $\mathcal{U}(-\sqrt{1/N_{\ell-1}},\sqrt{1/N_{\ell-1}})$, where $N_{\ell-1}$ is the width of the $\ell-1$-th layer. Note that the network with RAFs can be expressed by a network with a single activation function in each neuron but different neurons can use different activation functions. For example, in the case of poly-sine-Gaussian networks, we will use $1/4$ neurons within each layer with $x$ activation function, $1/4$ with $x^2$, $1/4$ with $\sin(x)$, and $1/4$ with $\text{exp}(-x^2)$. 
In this new setting, it is not necessary to train extra combination coefficients in the RAF. Though training the scaling parameters in the RAF might be beneficial in general applications, we focus on justifying the poly-sine-Gaussian activation function without emphasizing the scaling parameters. Hence, in almost all tests, the scaling parameters are set to be 1 for $x$, $x^2$, and $\sin(x)$, and the scaling parameter is set to be $0.1$ for $\text{exp}(-x^2)$. In the case of oscillatory 
solution~\eqref{eqn:oscillation}, we specify the scaling parameter of $\sin(x)$ to introduce oscillation in the NTK. The idea of scaling parameters was also tested and verified in \cite{JAGTAP2020109136,Ameya}. The other implementation detail can be found in Appendix. 
\subsubsection{Discontinuous Function Regression}
We first show the advantage of the poly-sine-Gaussian activation function by regression a discontinuous function, $f(x) = -2x+1, $ when $x\geq 0$, $f(x) = -2x-1, $ when $x<0$, 
on the domain $\Omega = [-1,1]$. The relative $L^2$ error is presented in Table \ref{tab:pdes} and the training process is visualized in Figure \ref{fig:regressions}. The regression result shows that the poly-sine-Gaussian activation has the best performance. We would like to remark that rational activation \cite{boulle2020rational} works well for regression problems but fails in our PDE problems without meaningful solutions. Hence, we only compare RAFs with rational activation functions in this example. The numerical results justify the combination of four kinds of activation functions. Remark that the computational time of rational activation functions is twice of the time of RAFs, though their accuracy is almost the same.

\begin{table*}
  \begin{minipage}[t]{0.45\textwidth}
  \captionof{table}{The best historical accuracy for the  equation in \eqref{eqn:oscillation} with an oscillatory solution when the scaling parameters of $\sin(x)$ activation functions either in the first hidden layer or the last hidden layer are pre-fixed.}
  \scalebox{0.9}{
    \begin{tabular}{ccc}
    \toprule
    Activation & Position & L2 error \\
    \midrule
      $x \oplus x^2 \oplus \sin(x)$ & first & 2.31 e-07 \\
      $x \oplus x^2 \oplus \sin(x)$ & last & 3,14 e-06  \\
      $x \oplus x^2 \oplus \sin(x) \oplus \cos(x) $ & first & 3.79 e-07 \\
	  $x \oplus x^2 \oplus \sin(x) \oplus \cos(x) $ & last & 1.90 e-03\\
    \bottomrule
    \end{tabular}%
}
\label{tbl:oscillation2}
\end{minipage}
\hspace{0.4cm}
  \begin{minipage}[t]{0.5\textwidth}
  \centering
  \captionof{table}{The condition number of the NTK in \eqref{eqn:NTKPDE} of PDE solvers with different activation functions at initialization. The NTK matrix is evaluated with $100$ samples, i.e., the matrix size is $100\times 100$.}
  \scalebox{0.85}{
\begin{tabular}{cccc}
    \toprule
    Activation & Eqn. \eqref{eqn:regular} & Eqn. \eqref{eqn:singular}  & Eqn. \eqref{eqn:oscillation} \\
    \midrule
    $\text{ReLU} ^3$  &  1.32 e+11  &  2.60 e+10  &   3.23 e+11 \\
    	$x \oplus x^2$   &  4.71 e+11  &  1.74 e+11  &   4.28 e+11 \\
    	$x \oplus x^2 \oplus \text{ReLU}$  &  1.01 e+11  &  1.09 e+10  &  3.11 e+10 \\
    	$x \oplus x^2 \oplus \text{ReLU}^3$    &  2.03 e+12  &  1.65 e+11  &   3.45 e+11 \\
    	$x \oplus x^2 \oplus \sin(x)$   &  1.92 e+12  &  5.18 e+10  &  1.10 e+11 \\
    	$x \oplus x^2 \oplus \sin(x) \oplus$ Gaussian   &   3.91 e+08  & 4.11 e+09  &   1.36 e+10 \\
    \bottomrule
\end{tabular}%
}
\label{tbl:NTK}
    \end{minipage}
\end{table*}

\subsubsection{Poisson Equation with a Smooth Solution}\label{equ:Construction} 

Now we solve a two-dimensional Poisson equation 
\begin{equation}
\label{eqn:regular}
	\begin{aligned}
		-\Delta u = f \text{ for }  \Bx \in \Omega \text{ and } u =0  \text{ for } \Bx \in \partial \Omega
	\end{aligned}
\end{equation}
with a smooth solution $u(\Bx) = x_1^2(1-x_1) x_2^2(1-x_2)$ defined on $\Omega=[0,1]^2$.  The numerical solution can be constructed as 
$
        \hat{u}(\Bx;\bm{\theta}) = \left(\Pi_{i=1}^2 x_i(1-x_i)\right) \phi(\Bx;\bm{\theta}),
$
where $\phi(\Bx;\bm{\theta})$ is an NN. We apply the loss function~\eqref{eqn:dloss2} to identify an estimated solution. The relative $L^2$ errors for different activation functions are shown in Table \ref{tab:pdes} and the corresponding training process is visualized in Figure \ref{fig:poisson}. The RAF with $\mathcal{A}=\{x,x^2,\sin(x),\text{exp}(-x^2)\}$ achieves the best performance. The networks with other activation functions reach local minimal and cannot escape from these minimal after $20k$ iterations, while the poly-sine-Gaussian network continuously reduces the error even after $50k$ iterations. The numerical results also justify the combination of four kinds of activation functions. 



\subsubsection{PDE with Low Regularity}
Next, we consider a two-dimensional PDE
\begin{equation}
\label{eqn:singular}
-\nabla\cdot (|\Bx| \nabla u) = f \text{ for }   \Bx \in \Omega \text{ and }  u =0  \text{ for }  \Bx \in \partial \Omega
\end{equation}
with a solution $u(\Bx) = \sin(2\pi (1- |\Bx| ))$ defined on $\Omega=\{\Bx : |\Bx| \leq 1\}$. The exact solution has low regularity at the origin. Let $\hat{u}(\Bx;\bm{\theta}) = (1-|\Bx|)\phi(\Bx;\bm{\theta})$, where $\phi(\Bx;\bm{\theta})$ is an NN and $\hat{u}(\Bx;\bm{\theta})$ satisfies the boundary condition automatically. The loss function~\eqref{eqn:dloss2} is used to identify an estimated solution to the equation \eqref{eqn:singular}. The relative $L^2$ errors for different activation functions are shown in Table \ref{tab:pdes} and the training curves is visualized in Figure \ref{fig:lowregularity}. Since the true solution has low regularity, it is more challenging than the example~\eqref{eqn:regular} to obtain good accuracy. The RAF with $\mathcal{A}=\{x,x^2,\sin(x),\text{exp}(-x^2)\}$ achieves lowest test error, which justifies the combination of four activation functions. 



\subsubsection{PDE with an Oscillatory Solution}
Next, to verify the performance of $\sin(x)$ in the RAF, we consider a two-dimensional Poisson equation as follows, 
\begin{equation}
\label{eqn:oscillation}
-\Delta u+ (u+2)^2 = f \text{ for }  \Bx \in \Omega 
\end{equation}
with a Dirichlet boundary condition and an oscillatory solution $u(\Bx) = \sin(6\pi x_1) \sin(6\pi x_2) $ defined on $\Omega=[0,1]^2$. The NN is constructed as in Section \ref{equ:Construction} with width $100$. The loss function~\eqref{eqn:dloss2} is used to identify the NN solution. The test error is shown in Table \ref{tab:pdes} and Figure \ref{fig:oscillation}. RAFs with $\{x,x^2,\sin(x),\text{exp}(-x^2)\}$ achieve the best performance. 

As discussed in Section \ref{sec:RAF}, introducing oscillation in NNs is crucial to lessen the spectral bias of NNs. Fixing different scaling parameters in $\sin(x)$ can help to lessen the spectral bias better and obtain high-resolution image 
reconstruction \cite{tancik2020fourier}. Therefore, in the case of oscillatory target functions, we also specify scaling parameters in $\sin(x)$ to verify the performance. If $n$ $\sin(x)$ functions are used in a layer, we will use $\{\sin(2\pi x),\sin(4\pi x),\dots,\sin(2n\pi x)\}$. Besides, it is also of interest to see the performance of $\cos(x)$. Since specifying a wide range of scaling parameters in every hidden layer will create too much oscillation, we only specify scaling parameters either in the first or the last hidden layer. Therefore, four tests were conducted and the results are shown in Table \ref{tbl:oscillation2} and Figure \ref{fig:oscillation2}. The results show that $\cos(x)$ does not have an effective gain, but specifying different scaling parameters improves the performance especially in the first hidden layer. Further, we test a super oscillatory solution $u(x) = \sin(40x_1)\sin(40x_2)$ to the equation \eqref{eqn:oscillation} and the result in Figure~\ref{fig:oscillation20} shows our methods outperform ReLU$^3$.

\subsubsection{Nonlinear Schr\"odinger Equation}
Last, we consider a $d$-dimensional nonlinear Schr\"odinger operator defined as $\mathcal{L}\varphi = - \Delta \varphi + \varphi^3 + V \varphi$ on $\Omega$, where $V(x) = -\frac{1}{c^2}\exp(\frac{2}{d}\sum_{i=1}^d\cos x_i)+ \sum_{i=1}^d(\frac{\sin^2 x_i}{d^2}- \frac{\cos x_i}{d})-3$ and $\Omega=[0,2\pi]^d$. $\lambda = -3$ and $\varphi(x) = \exp( \frac{1}{d}\sum_{j=1}^d\cos(xj ))/c$ is the leading eigenpair of the operator $\mathcal{L}$. Here $c$ is a positive
constant such that $\int_{\Omega}\varphi^2(x)dx=|\Omega|$. We follow the approach in \cite{HAN2020} to solve for the leading eigenpair. The NN in \cite{HAN2020} consists of two parts: 1) the first hidden layer uses $\sin(x)$ and $\cos(x)$ with different frequencies so that the whole network satisfies periodic boundary conditions; 2) the other hidden layers uses ReLU activation functions. We compare three activation functions, ReLU, ReLU$^3$, poly-sine-Gaussian, after the first hidden layer. Table~\ref{tab:pdes} shows the error for different activation function and Figure \ref{fig:eigen5} and \ref{fig:eigen10} display the training curves for $d=5$ and $d=10$, respectively. One can see poly-sine-Gaussian reaches a smaller minimal than ReLU, ReLU$^3$.

\subsubsection{Neural Tangent Kernel of PDE Solvers}
As discussed in Section \ref{sec:related}, the condition number of NTK is also a crucial factor that determines the performance of deep learning. The condition numbers of NTK for the different PDE problems at initialization when different activation functions are used are summarized in Table \ref{tbl:NTK}. The condition number of the poly-sine-Gaussian activation function is smallest. Hence, from the perspective of NTK, we have also justified the combination of basic activation functions in the poly-sine-Gaussian activation function.

\section{Conclusion}\label{sec:cond}
 We propose RAF and its approximation theory. NNs with this activation function can reproduce traditional approximation tools (e.g., polynomials, Fourier basis functions, wavelets, radial basis functions) and approximate a certain class of functions with exponential and dimension-independent approximation rates. We have numerically demonstrated that RAFs can generate neural tangent kernels with a better condition number than traditional activation functions, lessening the spectral bias of deep learning. Extensive experiments on coordinate-based data representation and PDEs demonstrate the effectiveness of the proposed activation function. We have not explored the optimal choice of basic activation functions in this paper, which would be problem-dependent and is left for future work.

{\bf Acknowledgements.} C. W. was partially supported by National Science Foundation Award DMS-1849483. H. Y. was partially supported by the US National Science Foundation under award DMS-1945029. The authors thank Mo Zhou for sharing his code for Schr{\"o}dinger equations.

\bibliographystyle{icml2021}

\onecolumn
\appendix
{
\section{Preliminaries}
\subsection{Deep Neural Networks}\label{sec:dnn}
Mathematically, NNs are a form of highly non-linear function parametrization via function compositions using simple non-linear functions \cite{IanYoshuaAaron2016}.  The justification of this kind of approximation is given by the universal approximation theorems of NNs in \cite{kurkova1992,barron1993,yarotsky2017,yarotsky2018} with newly developed quantitative and explicit error characterization \cite{Shen2,Shen3,Shen4}, which shows that function compositions are more powerful than other traditional approximation tools. There are two popular neural network structures used in NN-based PDE solvers.

The first one is the fully connected feed-forward neural network (FNN), which is the composition of $L$ simple nonlinear functions as follows:
\begin{equation}\label{eqn:FNN}
	\phi(\bm{x};\bm{\theta}):=\bm{a}^T \bm{h}_L \circ \bm{h}_{L-1} \circ \cdots \circ \bm{h}_{1}(\bm{x}),
\end{equation}
 where $\bm{h}_{\ell}(\bm{x})=\sigma\left(\bm{W}_\ell \bm{x} + \bm{b}_\ell \right)$ with $\bm{W}_\ell \in \mathbb{R}^{N_{\ell}\times N_{\ell-1}}$, $\bm{b}_\ell \in \mathbb{R}^{N_\ell}$ for $\ell=1,\dots,L$, $\bm{a}\in \mathbb{R}^{N_L}$, $\sigma$ is a non-linear activation function, e.g., a rectified linear unit (ReLU) $\sigma(x)=\max\{x,0\}$ or hyperbolic tangent function $\tanh(x)$. Each $\bm{h}_\ell$ is referred as a hidden layer,  $N_\ell$ is the width of the $\ell$-th layer, and $L$ is called the depth of the FNN. In the above formulation, $\bm{\theta}:=\{\bm{a},\,\bm{W}_\ell,\,\bm{b}_\ell:1\leq \ell\leq L\}$ denotes the set of all parameters in $\phi$, which uniquely determines the underlying neural network.

Another popular network is the residual neural network (ResNet) introduced in \cite{HeZhangRenSun2016}. We present its variant defined recursively as follows:
 \begin{eqnarray}\label{eqn:ResNet}
 \bm{h}_0&=&\bm{V}\bm{x},\nonumber\\
 \bm{g}_\ell&=&\sigma(\bm{W}_\ell\bm{h}_{\ell-1}+\bm{b}_{\ell}),
\qquad\ell=1,2,\dots,L,\nonumber\\
 \bm{h}_\ell&=&\bm{\bar{U}}_\ell \bm{h}_{\ell-2}+\bm{U}_\ell\bm{g}_\ell, \quad\ell=1,2,\dots,L,\nonumber\\
\phi(\bm{x};
 \bm{\theta})&=&\bm{a}^T\bm{h}_L,
 \end{eqnarray}
 where $\bm{V}\in \mathbb{R}^{N_0\times d}$, $\bm{W}_{\ell}\in \mathbb R^{N_{\ell}\times N_0}$, 
$\tilde{\bm{U}}_{\ell}\in \mathbb R^{N_0\times N_0}$, $\bm{U}_{\ell}\in \mathbb R^{N_0\times N_{\ell}}$, $\bm{b}_{{\ell}}\in \mathbb R^{N_{\ell}}$ for ${\ell}=1, \cdots, L$, $\bm{a}\in \mathbb R^{N_0}$, $\bm{h}_{-1}=0$. Throughout this paper, we consider $N_0=N_{\ell}=N$ and $\bm{U}_{\ell}$ is set as the identity matrix in the numerical implementation of ResNets for the purpose of simplicity. Furthermore, as used in \cite{EYu2018}, we set $\tilde{\bm{U}}_{\ell}$ as the identity matrix when $\ell$ is even and  set $\tilde{\bm{U}}_{\ell}=0$ when $\ell$ is odd. 

\subsection{Deep Learning for Regression Problems} 
Regression problems aim at identifying an unknown target function $f:\mathbf{x}\in\Omega \rightarrow y{\in \mathbb{R}}$ from training samples $\{(\mathbf{x}_i,y_i)\}_{i=1}^N$, where $\mathbf{x}_i$'s are usually assumed to be i.i.d samples from an underlying distribution $\pi$ {defined on a domain $\Omega\subseteq\mathbb{R}^n$,} and $y_i=f(\mathbf{x}_i)$ (probably with an additive noise). Consider the square loss $\ell(\mathbf{x},y;\bm{\theta})=\left| \phi(\mathbf{x};\bm{\theta})-y\right|^2$ of a given NN $\phi(\mathbf{x};\bm{\theta})$ that is used to approximate $f(\mathbf{x})$,   the population risk (error) and empirical risk (error) functions are respectively
\begin{equation}\label{eqn:pop}
\mathcal{J}(\bm{\theta})=\frac{1}{2}\E_{\mathbf{x}\sim \pi}\left[ \left| \phi(\mathbf{x};\bm{\theta})-f(\mathbf{x})\right|^2 \right],\quad \hat{\mathcal{J}}(\bm{\theta})=\frac{1}{2N}\sum_{i=1}^N \left| \phi(\mathbf{x}_i;\bm{\theta})-y_i\right|^2,
\end{equation}
which are also functions that depend on the depth $L$ and width $N_\ell$ of $\phi$ implicitly. The optimal set $\hat{\bm{\theta}}$ is identified via
\begin{equation}\label{eqn:RGloss}
\hat{\bm{\theta}}=\argmin_{\bm{\theta}} \hat{ \mathcal{J}}(\bm{\theta}),
\end{equation}
and $\phi(\cdot;\hat{\bm{\theta}}):\Omega\to\mathbb{R}$ is the learned NN that {approximates} the unknown function $f$.

\subsection{Deep Learning for Solving PDEs}\label{sec:LSM}

Deep learning can be applied to solve various PDEs including the initial value problems and boundary value problems (BVP) based on different variational formulations \cite{doi:10.1002/cnm.1640100303,712178,EYu2018,liao2019deep}. In this paper, we will take the example of BVP and the least squares method (LSM) \cite{doi:10.1002/cnm.1640100303,712178} without loss of generality. The generalization to other problems and methods is similar. Consider the BVP
\begin{equation}\label{eqn:BVP}
\begin{split}
&\mathcal{D}u(\bm{x})=f(u(\bm{x}),\bm{x}),\text{~in~}\Omega,\\
&\mathcal{B}u(\bm{x})=g(\bm{x}),\text{~on~}\partial\Omega,
\end{split}
\end{equation}
where $\mathcal{D}: \Omega\rightarrow \Omega$ is a differential operator that can be nonlinear, $f(u(\bm{x}),\bm{x})$ can be a nonlinear function in $u$, $\Omega$ is a bounded domain in $\mathbb{R}^d$, and $\mathcal{B}u=g$ characterizes the boundary condition. Other types of problems like initial value problems can also be formulated as a BVP as discussed in \cite{Gu2020}. Then LSM seeks a solution $u(\bm{x};\bm{\theta})$ as a neural network with a parameter set $\bm{\theta}$ via the following optimization problem
\begin{equation}\label{eqn:loss}
\underset{\bm{\theta}}{\min}~\mathcal{L}(\bm{\theta}):=\|\mathcal{D}u(\bm{x};\bm{\theta})-f(u,\bm{x})\|_{L^2(\Omega)}^2+\lambda\|\mathcal{B}u(\bm{x};\bm{\theta})-g(\bm{x})\|_{L^2(\partial\Omega)}^2,
\end{equation}
where $\mathcal{L}$ is the loss function consisting of the $L^2$-norm of the PDE residual $\mathcal{D}u(\bm{x};\bm{\theta})-f(u,\bm{x})$ and the boundary residual $\mathcal{B}u(\bm{x};\bm{\theta})-g(\bm{x})$, and $\lambda>0$ is a regularization parameter. 

The goal of \eqref{eqn:loss} is to find an appropriate set of parameters $\bm{\theta}$ such that the NN $u(\bm{x};\bm{\theta})$ minimizes the loss $\mathcal{L}(\bm{\theta})$. If the loss $\mathcal{L}(\bm{\theta})$ is minimized to zero with some $\bm{\theta}$, then $u(\bm{x};\bm{\theta})$ satisfies $\mathcal{D}u(\bm{x};\bm{\theta})-f(\bm{x})=0$ in $\Omega$ and $\mathcal{B}u(\bm{x};\bm{\theta})-g(\bm{x})=0$ on $\partial\Omega$, implying that $u(\bm{x};\bm{\theta})$ is exactly a solution of \eqref{eqn:BVP}. If $\mathcal{L}$ is minimized to a nonzero but small positive number, $u(\bm{x};\bm{\theta})$ is close to the true solution as long as \eqref{eqn:BVP} is well-posed (e.g. the elliptic PDE with Neumann boundary condition, see Thm.~4.1 in \cite{Gu2020}).

In the implementation of LSM, the minimization problem in \eqref{eqn:loss} is solved by SGD or its variants (e.g. Adagrad \cite{Duchi2011}, Adam \cite{Kingma2014} and AMSGrad \cite{Reddi2019}). In each iteration of the SGD, a stochastic loss function defined below is minimized instead of the original loss function in \eqref{eqn:loss}:
\begin{equation}\label{eqn:dloss}
    \begin{aligned}
    \min_{\bm{\theta}}\hat{\mathcal{L}} (\bm{\theta}):= \frac{1}{N}\sum_{i=1}^N   \big(\mathcal{D}u(\bm{x}_i; \bm{\theta}) - f(\bm{x}_i)\big)^2   + \frac{1}{M} \lambda\sum_{j=1}^{M} \big(\mathcal{B}u(\bm{x}_j; \bm{\theta})-g(\bm{x}_j)\big)^2,
    \end{aligned}
\end{equation}
where $\{\bm{x}_i\}_{i=1}^N$ are $N$ uniformly sampled random points in $\Omega$ and $\{\bm{x}_j\}_{j=1}^{M}$ are $M$ uniformly sampled random points on $\partial \Omega$. These random samples will be renewed in each iteration. Throughout this paper, we will use Adam, which is a variant of SGD based on momentum, to solve the NN-based optimization.

To facilitate the optimization convergence to the desired PDE solution,
special network structures can be proposed such that the NN can satisfy common boundary conditions, which can simplify the loss function in \eqref{eqn:loss} to
\begin{equation}\label{eqn:loss2}
\underset{\bm{\theta}}{\min}~\mathcal{L}(\bm{\theta}):=\|\mathcal{D}u(\bm{x};\bm{\theta})-f(u,\bm{x})\|_{L^2(\Omega)}^2,
\end{equation}
since $\mathcal{B}u(\bm{x}; \bm{\theta})=g(\bm{x})$ is satisfied by construction. Correspondingly, the stochastic loss function is reduced to 
\begin{equation}\label{eqn:dloss2}
    \begin{aligned}
    \min_{\bm{\theta}}\hat{\mathcal{L}} (\bm{\theta}):= \frac{1}{N}\sum_{i=1}^N   \big(\mathcal{D}u(\bm{x}_i; \bm{\theta}) - f(\bm{x}_i)\big)^2 .
    \end{aligned}
\end{equation}
In numerical implementation, the LSM loss function in \eqref{eqn:loss2} is more attractive because \eqref{eqn:loss} heavily relies on the selection of a suitable weight parameter $\lambda$ and a suitable initial guess. If $\lambda$ is not appropriate, it may be difficult to identify a reasonably good minimizer of \eqref{eqn:loss}, as shown by extensive numerical experiments in \cite{712178,gu2020structure,Lyu2020EnforcingEB}. However, we would like to remark that it is difficult to build NNs that automatically satisfy complicated boundary conditions especially when the domain $\Omega$ is irregular. 

The design of these special NNs depends on the type of boundary conditions. We will discuss the case of Dirichlet boundary conditions by taking one-dimensional problems defined in the domain $\Omega=[a, b]$ as an example. Network structures for more complicated boundary conditions in high-dimensional domains can be constructed similarly. The reader is referred to \cite{gu2020structure,Lyu2020EnforcingEB} for other kinds of boundary conditions. 

Suppose $\hat{u}(x;\bm{\theta})$ is a generic NN with trainable parameters $\bm{\theta}$. We will augment $\hat{u}(x;\bm{\theta})$ with several specially designed functions to obtain a final network ${u}(x;\bm{\theta})$ that satisfies $\mathcal{B}u({x};\bm{\theta})=g({x})$ automatically. For simplicity, let us consider the boundary conditions $u(a)=a_0$ and $u(b)=b_0$. In this case, we can introduce two special functions $h(x)$ and $l(x)$ to augment $\hat{u}(x;\bm{\theta})$ to obtain the final network $u(x;\bm{\theta})$:
\begin{equation}\label{11}
u(x;\bm{\theta}) = h(x)\hat{u}(x;\bm{\theta})+l(x).
\end{equation}
Then $u(x;\bm{\theta})$ is used to approximate the true solution of the PDE and is trained through \eqref{eqn:loss2}.

 A straightforward choice for $l(x)$ is 
\begin{equation*}
l(x)=(b_0-a_0)(x-a)/(b-a)+a_0,
\end{equation*}
and $h(x)$ can be set as
\begin{equation*}
h(x) = (x-a)^{p_a}(x-b)^{p_b},
\end{equation*}
with $0<p_a,~p_b\leq1$. To obtain an accurate approximation, $p_a$ and $p_b$ should be chosen to be consistent with the orders of $a$ and $b$ of the true solution, hence no singularity will be brought into the network structure. 

\subsection{The Training Behavior of Deep Learning}
\label{sec:FP}

The least-squares optimization problems in \eqref{eqn:loss} and \eqref{eqn:loss2} are highly non-convex and hence they are challenging to solve. For regression problems or solving linear PDEs, under the assumption of over-parameterized NNs (i.e., the width of NNs is sufficiently large) and appropriate random initialization of NN parameters, it was shown that the least-squares optimization admits global convergence by gradient descent with a linear convergence rate \cite{DBLP:journals/corr/abs-1806-07572,Du2018,Zhu2019,Chen1,LuoYang2020}. Though the over-parametrization assumption might not be realistic, it is still a positive sign for the justification of NNs in these least-squares problems. However, the convergence rate depends on the spectrum of the target function. The training of a randomly initialized NN has a stronger preference for reducing the fitting error of low-frequency components of a target solution. The high-frequency component of the target function would not be well captured until the low-frequency error has been eliminated. This phenomenon is called the F-principle in \cite{FP} and the spectral bias of deep learning in \cite{cao2019towards}. Related works on the learning behavior of NNs in the frequency domain is further investigated in \cite{FP,Luo2019}. In the case of nonlinear PDEs, these theoretical works imply that NN-based solvers would also have a bias towards reducing low-frequency errors \cite{wang2020pinns}. Without the assumption of over-parametrization, to the best of our knowledge, there is no theoretical guarantee that NN-based PDE solvers can identify the global minimizer via a standard SGD. Through the analysis of
the optimization energy landscape of SGD without the over-parameterization, it was shown that SGD with small batches tends to converge to the flattest minimum \cite{neyshabur2017geometry,lei2018implicit,dai2018towards}. However, such local minimizers might not give the desired PDE solutions. Hence, designing new training techniques to make SGD capable of identifying better minimizers has been an active research field.

\subsection{Neural Tangent Kernel}\label{sec:NTK}

Neural tangent kernel (NTK) originally introduced in \cite{DBLP:journals/corr/abs-1806-07572} and further investigated in \cite{arora2019finegrained,Lee_2020,cao2019towards,LuoYang2020,wang2020pinns} is one of the popular tools to study the training behavior of deep learning in regression problems and PDE problems. Let us briefly introduce the main idea of NTK following the linearized model for regression problems in \cite{Lee_2020} for simplicity. This introduction is sufficient for us to discuss the advantage of RAFs later in the next section. 

Let us use $\mathcal{X}$ to denote the set of training sample locations $\{\bm{x}_i\}_{i=1}^N$ in the empirical loss function $\hat{\mathcal{J}}(\bm{\theta})$ in \eqref{eqn:pop}. Let $\mathcal{Y}$ be the set of function values at these sample locations. Using gradient flow to analyze the training dynamics of $\hat{\mathcal{J}}(\bm{\theta})$, we have the following evolution equations:
\begin{equation}\label{eqn:gdf}
\dot{\bm{\theta}}_t = -\nabla_{\bm{\theta}}\phi_t(\mathcal{X})^T\nabla_{\phi_t(\mathcal{X})}\hat{\mathcal{J}},
\end{equation}
and
\begin{equation}\label{eqn:gdff}
\dot{\phi}_t(\mathcal{X})=\nabla_{\bm{\theta}}\phi_t(\mathcal{X})\dot{\bm{\theta}}_t=-\hat{\Theta}_t(\mathcal{X},\mathcal{X})\nabla_{\phi_t(\mathcal{X})}\hat{\mathcal{J}},
\end{equation}
where $\bm{\theta}_t$ is the parameter set at iteration time $t$, $\phi_t(\mathcal{X})=\text{vec}([\phi_t(\bm{x};\bm{\theta}_t)]_{\bm{x}\in\mathcal{X}})$ is the $N\times 1$ vector of concatenated function values for all samples, and $\nabla_{\phi_t(\mathcal{X})}\hat{\mathcal{J}}$ is the gradient of the loss with respect to the network output vector $\phi_t(\mathcal{X})$, $\hat{\Theta}_t:=\hat{\Theta}_t(\mathcal{X},\mathcal{X})$ in $\mathbb{R}^{N\times N}$ is the NTK at iteration time $t$ defined by
\[
\hat{\Theta}_t=\nabla_{\bm{\theta}}\phi_t(\mathcal{X})\nabla_{\bm{\theta}}\phi_t(\mathcal{X})^T.
\]
The NTK can also be defined for general arguments, e.g., $\hat{\Theta}_t(\bm{x},\mathcal{X})$ with $\bm{x}$ as a test sample location. 

After initialization, the training dynamics of deep learning can be characterized by \eqref{eqn:gdf} and \eqref{eqn:gdff}. The steady-state solutions of these evolution equations give the learned network parameters and the learned neural network in the regression problem. However, these evolution equations are highly nonlinear and it is difficult to obtain the explicit formulations of their solutions. Fortunately, as discussed in the literature \cite{DBLP:journals/corr/abs-1806-07572,arora2019finegrained,Lee_2020,cao2019towards,LuoYang2020}, when the network width goes to infinity, these evolution equations can be approximately characterized by their linearization, the solution of which admit simple explicit formulas.

For simplicity, we consider the linearization in \cite{Lee_2020} to obtain explicit solutions to discuss the training dynamics of deep learning. In particular, the following linearized network by Taylor expansion is considered,
\begin{equation}\label{eqn:lin}
\Li(\bm{x}):=\phi(\bm{x};\bm{\theta}_0) + \nabla_{\bm{\theta}} \phi(\bm{x};\bm{\theta}_0) \bm{\omega}_t,
\end{equation}
where $\bm{\omega}_t:=\bm{\theta}_t-\bm{\theta}_0$ is the change in the parameters from their initial values. The dynamics of gradient flow using this linearized function are governed by
\begin{equation}\label{eqn:gdfl}
\dot{\bm{\omega}}_t = -\nabla_{\bm{\theta}}\phi_0(\mathcal{X})^T\nabla_{\Li(\mathcal{X})}\hat{\mathcal{J}},
\end{equation}
and
\begin{equation}\label{eqn:gdffl}
\dot{\phi}^{\text{lin}}_t(\bm{x})=-\hat{\Theta}_0(\bm{x},\mathcal{X})\nabla_{\Li(\mathcal{X})}\hat{\mathcal{J}}.
\end{equation}
The above evolution equations have closed form solutions
\[
\bm{\omega}_t = -\nabla_{\bm{\theta}}\phi_0(\mathcal{X})^T \hat{\Theta}^{-1}_0 \left(I-e^{-\hat{\Theta}_0 t}\right)(\phi_0(\mathcal{X})-\mathcal{Y}),
\]
and
\begin{equation}\label{eqn:evs0}
\Li(\mathcal{X}) = \left(I-e^{-\hat{\Theta}_0 t}\right)\mathcal{Y}+ e^{-\hat{\Theta}_0 t} \phi_0(\mathcal{X}).
\end{equation}
For an arbitrary point $\bm{x}$,
\begin{equation}\label{eqn:evs}
\Li(\bm{x}) =\phi_0(\bm{x})- \hat{\Theta}_0(\bm{x},\mathcal{X})\hat{\Theta}_0^{-1}\left(I-e^{-\hat{\Theta}_0 t}\right)(\phi_0(\mathcal{X})-\mathcal{Y}),
\end{equation}
which is equivalent to
\begin{equation}\label{eqn:evs2}
\Li(\bm{x}) - \phi_0(\bm{x})= \hat{\Theta}_0(\bm{x},\mathcal{X})\hat{\Theta}_0^{-1}\left(I-e^{-\hat{\Theta}_0 t}\right)(\mathcal{Y}-\phi_0(\mathcal{X})).
\end{equation}

Therefore, once the initialized network $\phi_0(\bm{x})$ and the NTK at initialization $\hat{\Theta}_0$ are computed, we can obtain the time evolution of the linearized neural network without running gradient descent. The solution in \eqref{eqn:evs} serves as an approximate solution to the nonlinear evolution equation in \eqref{eqn:gdff}. Based on \eqref{eqn:evs2}, we see that deep learning can be approximated by a kernel method with the NTK $\hat{\Theta}_0$ that updates the initial prediction $\phi_0(\bm{x})$ to a correct one.

There mainly two kinds of observations from \eqref{eqn:evs} from the perspective of kernel methods. The first one is through the eigendecomposition of the initial NTK. If the initial NTK is positive definite, $\Li$ will eventually converge to a neural network that fits all training examples and its generalization capacity is similar to kernel regression by \eqref{eqn:evs}. The error of $\Li$ along the direction of eigenvectors of $\hat{\Theta}_0$ corresponding to large eigenvalues decays much faster than the error along the direction of eigenvectors of small eigenvalues, which is referred to as the spectral bias of deep learning. The second one is through the condition number of the initial NTK. Since NTK is real symmetric, its condition number is equal to its largest eigenvalue over its smallest eigenvalue. If the initial NTK is positive definite, in the ideal case when $t$ goes to infinity, $\left(I-e^{-\hat{\Theta}_0 t}\right)(\phi_0(\mathcal{X})-\mathcal{Y})$ in \eqref{eqn:evs} approaches to $\phi_0(\mathcal{X})-\mathcal{Y}$ and, hence, $\Li(\bm{x})$ goes to the desired function value for $\bm{x}\in\mathcal{X}$. However, in practice, when $\hat{\Theta}_0$ is very ill-conditioned, a small approximation error in $\left(I-e^{-\hat{\Theta}_0 t}\right)(\phi_0(\mathcal{X})-\mathcal{Y})\approx \phi_0(\mathcal{X})-\mathcal{Y}$ may be amplified significantly, resulting in a poor accuracy for $\Li(\bm{x})$ to solve the regression problem. We will discuss the advantage of the proposed RAFs in terms of these two observations later in the next two sections. 

The above discussion is for the NTK in regression setting. In the case of PDE solvers, we introduce the NTK below
\begin{equation}\label{eqn:NTKPDE}
\hat{\Theta}_t=\left(\nabla_{\bm{\theta}}\mathcal{D}\phi_t(\mathcal{X})\right)\left(\nabla_{\bm{\theta}}\mathcal{D}\phi_t(\mathcal{X})\right)^T,
\end{equation}
where $\mathcal{D}$ is the differential operator of the PDE. Similar to the discussion for regression problems, the spectral bias and the conditioning issue also exist in deep learning based PDE solvers by almost the same arguments.

\section{Proof of Theories}
\subsection{Proof of Theorem~\ref{thm:NNex2}}

The proof of Theorem~\ref{thm:NNex2} relies on the following lemma.

\begin{lemma}\label{lem:NNex2}
\begin{enumerate}[label=(\roman*)]
\item An identity map in $\R^d$ can be realized exactly by a poly-sine-Gaussian network with one hidden layer and $d$ neurons.
\item $f(x)=x^2$ can be realized exactly by a poly-sine-Gaussian network with one hidden layer and one neuron.
\item $f(x,y)=xy=\frac{(x+y)^2-(x-y)^2}{4}$ can be realized exactly by a poly-sine-Gaussian network with one hidden layer and two neurons.
\item Assume $P(\xx)=\xx^\xal=x_1^{\alpha_1}x_2^{\alpha_2}\cdots x_d^{\alpha_d}$ for $\xal\in \N^d$. For any $N,L\in \N^+$ such that $NL+2^{\lfloor \log_2 N \rfloor}\geq |\xal|$, there exists a poly-sine-Gaussian network $\phi$ with width $2N+d$ and depth $L+\lceil \log_2 N\rceil$ such that
    \begin{equation*}
    \phi(\xx)=P(\xx)\quad \text{for any $\xx\in \R^d$.}
    \end{equation*} 
\end{enumerate}
\end{lemma}

\begin{proof}
Part (i) to (iii) are trivial. We will only prove Part (iv). In the case of $|\xal|=k\leq 1$, the proof is simple and left for the reader. When $|\xal|=k\ge 2$, the main idea of the proof of (v) can be summarized in Figure \ref{fig:NN1}. By Part (i), we can apply a  poly-sine-Gaussian network to implement a $d$-dimensional identity map. This identity map maintains necessary entries of $\xx$ to be multiplied together. We apply poly-sine-Gaussian networks to implement the multiplication function in Part (iii) and carry out the multiplication $N$ times per layer. After $L$ layers, there are $k-NL\leq N$ multiplications to be implemented. Finally, these at most $N$ multiplications can be carried out with a small poly-sine-Gaussian network in a dyadic tree structure.
\end{proof}

Now we are ready to prove Thm.~\ref{thm:NNex2}.

\begin{proof}
The main idea of the proof is to apply Part (iv) of Lem.~\ref{lem:NNex2} $J$ times to construct $J$ poly-sine-Gaussian networks, $\{\phi_j(\xx)\}_{j=1}^J$, to represent $\xx^{\xal_j}$ and arrange these poly-sine-Gaussian networks as subnetwork blocks to form a larger poly-sine-Gaussian network $\tilde{\phi}(\xx)$ with $ab$ blocks as shown in Figure \ref{fig:NN2}, where each red rectangle represents one poly-sine-Gaussian network $\phi_j(\xx)$ and each blue rectangle represents one poly-sine-Gaussian network of width $1$ as an identity map of $\R$. There are $ab$ red blocks with $a$ rows and $b$ columns. When $ab\geq J$, these subnetwork blocks can carry out all monomials $\xx^{\xal_j}$. In each column, the results of the multiplications of $\xx^{\xal_j}$ are added up to the input of the narrow poly-sine-Gaussian network, which can carry the sum over to the next column. After the calculation of $b$ columns, $J$ additions of the monomials $\xx^{\xal_j}$ have been implemented, resulting in the output $P(\xx)$.

By Part (iv) of Lem.~\ref{lem:NNex2}, for any $N\in \N^+$, there exists a poly-sine-Gaussian network $\phi_j(\xx)$ of width $d+2N$ and depth $L_j = \lceil \frac{|\xal_j|}{N}\rceil +\lceil \log_2 N \rceil$ to implement $\xx^{\xal_j}$. Since $b\max_j L_j\leq b\left(  \frac{\max_j |\xal_j|}{N} +2 + \log_2 N \right)$, there exists a poly-sine-Gaussian network $\tilde{\phi}(\xx)$ of depth $b\left(  \frac{\max_j |\xal_j|}{N} + 2 + \log_2 N \right)$ and width $da+2Na+1$ to implement $P(\xx)$ as in Figure \ref{fig:NN2}. Note that the total width of each column of blocks is $ad+2Na+1$ but in fact this width can be reduced to $d+2Na+1$, since the red blocks in each column can share the same identity map of $\R^d$ (the blue part of Figure \ref{fig:NN1}).

Note that $b\left(  \frac{\max_j |\xal_j|}{N} + 2 + \log_2 N \right)\leq L$ is equivalent to $(L-2b-b\log_2 N)N\geq b \max_j |\xal_j|$. Hence, for any $N,L,a,b\in \N^+$ such that $ab\geq J$ and $(L-2b-b\log_2 N)N\geq b \max_j |\xal_j|$, there exists a poly-sine-Gaussian network $\phi(\xx)$ with width $2Na+d+1$ and depth $L$ such that $\tilde{\phi}(\xx)$ is a subnetwork of $\phi(\xx)$ in the sense of $\phi(\xx)=\text{Id}\circ\tilde{\phi}(\xx)$ with $\text{Id}$ as an identify map of $\R$, which means that $\phi(\xx)=\tilde{\phi}(\xx)=P(\xx)$. The proof of Part (v) is completed.
\end{proof}

	\begin{figure}[!ht]
		\centering
		\includegraphics[width=0.6\linewidth]{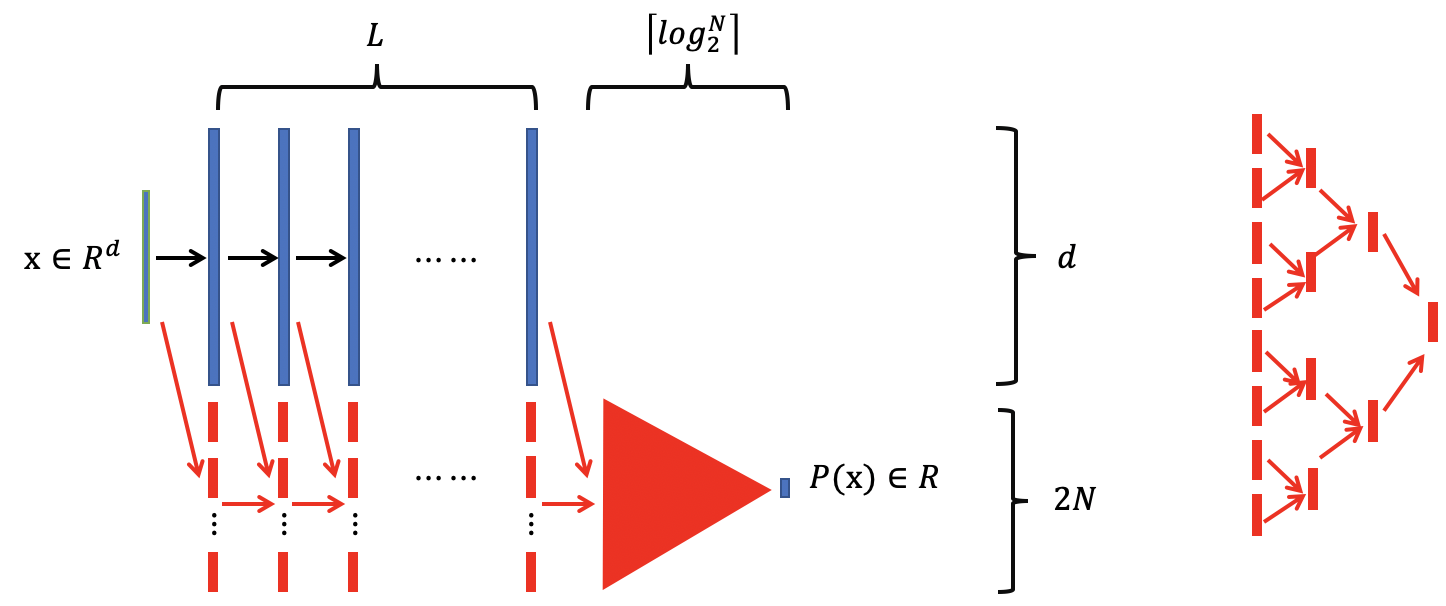}
		\caption{Left: An illustration of the proof of Lem.~\ref{lem:NNex2} (iv). Green vectors represent the input and output of the poly-sine-Gaussian network carrying out $P(\xx)$. Blue vectors represent the poly-sine-Gaussian network that implements a $d$-dimensional identity map in Part (i), which was repeatedly applied for $L$ times. Black arrows represent the data flow for carrying out the identity maps. Red vectors represent the poly-sine-Gaussian networks implementing the multiplication function in Part (iii) and there are $NL$ such red vectors. Red arrows represent the data flow for carrying out the multiplications. Finally, a red triangle represents a poly-sine-Gaussian network of width at most $2N$ and depth at most $\lceil \log_2^N \rceil$ carrying out the rest of the multiplications. Right: An example of the red triangle is given on the right when it consists of $15$ red vectors carrying out $15$ multiplications.}
		\label{fig:NN1}
	\end{figure}

	\begin{figure}[!ht]
		\centering
		\includegraphics[width=0.5\linewidth]{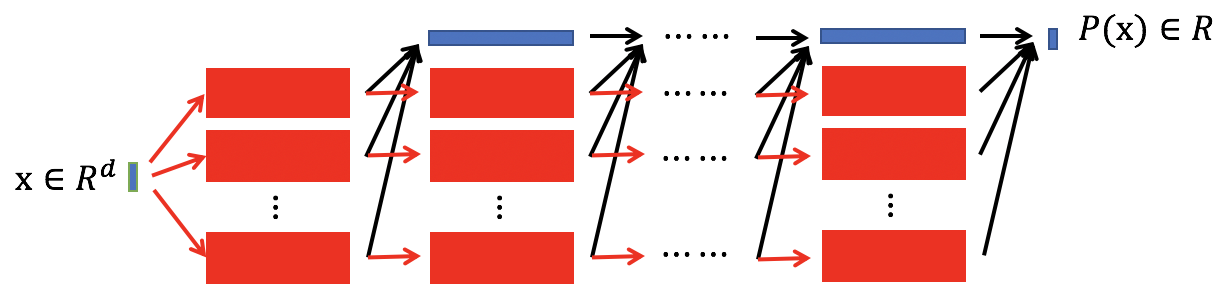}
		\caption{An illustration of the proof of Thm.~\ref{thm:NNex2}. Green vectors represent the input and output of the poly-sine-Gaussian network $\tilde{\phi}(\xx)$ carrying out $P(\xx)$. Each red rectangle represents one poly-sine-Gaussian network $\phi_j(\xx)$ and each blue rectangle represents one poly-sine-Gaussian network of width $1$ as an identity map of $\R$. There are $ab\geq J$ red blocks with $a$ rows and $b$ columns. When $ab\geq J$, these subnetwork blocks can carry out all monomials $\xx^{\xal_j}$. In each column, the results of the multiplications of $\xx^{\xal_j}$ are added up to (indicated by black arrows) the input of the narrow poly-sine-Gaussian network, which can carry the sum over to the next column. Each red arrow passes $\xx$ to the next red block. After the calculation of $b$ columns, $J$ additions of the monomials $\xx^{\xal_j}$ have been implemented, resulting in the output $P(\xx)$.}
		\label{fig:NN2}
	\end{figure}

\subsection{Proof of Thm.~\ref{thm:analytic}}
\begin{proof}
Let $M\geq1$, $s>1$, $C_f>0$ and $0<\epsilon<1$ be four scalars, and $f$ be an analytic function defined on $[-M,M]$ that is analytically continuable to the open Bernstein $s$-ellipse $E_{s}^M$, where it satisfies $\vert f(x)\vert\leq C_f$. We first approximate $f$ by a truncated Chebyshev series $f_n$, and then approximate $f_n$ by a poly-sine-Gaussian network $\phi$ using Thm.~\ref{thm:NNex2}.

Since $f$ is analytic in the open Bernstein $s$-ellipse $E_s^M$ then, for any integer $n\geq2$,
\begin{align*}
\left\Vert f_n(x)-f(x)\right\Vert_{L^\infty([-M,M])}\leq\frac{2C_fs^{-n}}{s-1} = \OO\left(C_fs^{-n}\right).
\end{align*}
Therefore, if we take $n=\OO\left(\frac{1}{\log_2s}\log_2\frac{2C_f}{\epsilon}\right)$, then the above term is bounded by $\epsilon$.

Let us now approximate $f_n$ by a poly-sine-Gaussian network $\phi$. We first write
\begin{align*}
f_n(x) = \sum_{k=0}^nc_kT_k\left(\frac{x}{M}\right), 
\end{align*}
with 
\begin{align}
\underset{0\leq k\leq n}{\max}\vert c_k\vert = \OO\left(C_fs\right),\;\text{via Thm.~8.1 in  \cite{trefethen2013}}.
\label{eq:coefficiens}
\end{align}
Since, $f_n$ is a polynomial of degree $n$, by Thm.~\ref{thm:NNex2} with $d=1$, $a=1$, and $b=n+1$, there exists a poly-sine-Gaussian network $\phi$ with width $2N+2$ and depth $L$ such that
\begin{align*}
\phi(x) = f_n(x)
\end{align*}
for $x\in\R$, as long as $N$ and $L$ satisfy $(L-2n-2-(n+1)\log_2 N)N\geq n(n+1)$.
This yields
\begin{align*}
\vert\phi(x) - f(x)\vert 
& =  \vert f_n(x) - f(x)\vert \leq  \epsilon.
\end{align*}
\end{proof}

\subsection{Proof of Thm.~\ref{thm:bandlimited}}
To show the approximation of poly-sine-Gaussian networks to generalized bandlimited functions, we will need Maurey's unpublished theorem below. It was used to study shallow network approximation by Barron in \cite{barron1993}. 

\begin{theorem}[Maurey's theorem]\label{thm:maurey}
Let $H$ be a Hilbert space with norm $\Vert\cdot\Vert$. 
Suppose there exists $G\subset H$ such that for every $g\in G$, $\Vert g\Vert\leq b$ for some $b>0$. 
Then, for every $f$ in the convex hull of $G$ and every integer $n\geq 1$, there is a $f_n$ in the convex hull of $n$ points in $G$ and a constant $c>b^2-\Vert f\Vert^2$ such that $\Vert f - f_n\Vert^2\leq \frac{c}{n}$.
\end{theorem}

\begin{proof}
Let $f$ be an arbitrary function in $\mathcal{H}_{K,M}$, and $\mu$ be an arbitrary measure. Let $F(\bm{w})=\vert F(\bm{w})\vert e^{i\theta(\bm{w})}$.
Since $f$ is real-valued, we may write
\begin{align*}
f(\bm{x}) & =  \mrm{Re}\,\Bigg(\int_{\R^d}C_F e^{i\theta(\bm{w})}K(\bm{w}\cdot\bm{x})\frac{\vert F(\bm{w})\vert}{C_F} d\bm{w}\Bigg), \\
& = \int_{[-M,M]^d}C_F\Bigg[\cos(\theta(\bm{w}))K_R(\bm{w}\cdot\bm{x})-\sin(\theta(\bm{w}))K_I(\bm{w}\cdot\bm{x})\Bigg]\frac{\vert F(\bm{w})\vert}{C_F} d\bm{w},
\end{align*}
{where $K_R(\bm{w}\cdot\bm{x})=\mrm{Re}(K(\bm{w}\cdot\bm{x}))$ and $K_I(\bm{w}\cdot\bm{x})=\mrm{Im}(K(\bm{w}\cdot\bm{x}))$.} The integral above represents $f$ as an infinite convex combination of functions in the set
\begin{align*}
G_{K,M}  = \Big\{\gamma\big[\cos(\beta)\mrm{Re}(K(\bm{w}\cdot\bm{x}))-\sin(\beta)\mrm{Im}(K(\bm{w}\cdot\bm{x}))\big],\,\vert\gamma\vert\leq C_F,\,\beta\in\R,\,\bm{w}\in[-M,M]^d\Big\}.
\end{align*}
Therefore, $f$ is in the closure of the convex hull of $G_{K,M}$.
Since functions in $G_{K,M}$ are bounded in the $L^2(\mu,B)$-norm by $2C_FD_K\sqrt{\mu(B)}\leq2C_F\sqrt{\mu(B)}$, Thm.~\ref{thm:maurey} tells us that there exist real coefficients $b_j$'s and $\beta_j$'s such that\footnote{We use Thm.~\ref{thm:maurey} with $b=2C_F\sqrt{\mu(B)}$, $c=b^2>b^2 - \Vert f\Vert^2$, and $\Vert\cdot\Vert=\Vert\cdot\Vert_{L^2(\mu, B)}$.} 
\begin{align*}
f_{\epsilon_0}(\bm{x}) 
= \sum_{j=1}^{\ceil{1/\epsilon_0^2}}b_j\big[\cos(\beta_j)K_R(\bm{w}\cdot\bm{x}) - \sin(\beta_j)K_I(\bm{w}\cdot\bm{x})\big],\quad\sum_{j=1}^{\ceil{1/\epsilon_0^2}}\vert b_j\vert \leq C_F,
\end{align*}
for some $0<\epsilon_0<1$ to be determined later, such that
\begin{align*}
\left\Vert f_{\epsilon_0}(\bm{x}) - f(\bm{x})\right\Vert_{L^2(\mu, B)}\leq 2C_F\sqrt{\mu(B)}\epsilon_0.
\end{align*}

We now approximate $f_{\epsilon_0}(\bm{x})$ by a poly-sine-Gaussian network $\phi(\bm{x})$. {Note that $K_R$ and $K_I$ are both analytic and satisfy the same assumptions as $K$. Using Theorem} \ref{thm:analytic}{, they can be approximated to accuracy $\epsilon_0$ using networks $\wKK_R$ and $\wKK_I$ of width and depth}
\begin{align*}
\OO\left(\frac{1}{\log_2s}\log_2\frac{C_K}{\epsilon_0}\right) \quad\text{and}\quad \OO\left(\left(\frac{1}{\log_2s}\log_2\frac{C_K}{\epsilon_0}\right) \log_2 \log_2\frac{C_K}{\epsilon_0}     \right),
\end{align*}
respectively. We define the poly-sine-Gaussian network $\phi(\bm{x})$ by
\begin{align*}
\phi(\bm{x}) = \sum_{j=1}^{\ceil{1/\epsilon_0^2}}b_j\big[\cos(\beta_j)\wKK_R(\bm{w}\cdot\bm{x}) - \sin(\beta_j)\wKK_I(\bm{w}\cdot\bm{x})\big].
\end{align*}
This network has width $\OO\left(\frac{1}{\epsilon_0^2\log_2s}\log_2\frac{C_K}{\epsilon_0}\right)$ and depth $\OO\left(\left(\frac{1}{\log_2s}\log_2\frac{C_K}{\epsilon_0}\right) \log_2 \log_2\frac{C_K}{\epsilon_0}     \right)$, and
\begin{align*}
\vert \phi(\bm{x}) - f_{\epsilon_0}(\bm{x})\vert 
& \leq\sum_{j=1}^{\ceil{\frac{1}{\epsilon_0^2}}}\vert b_j\vert\vert\wKK_R(\bm{w}_j\cdot \bm{x})-K_R(\bm{w}_j\cdot\bm{x})\vert
+ \sum_{j=1}^{\ceil{\frac{1}{\epsilon_0^2}}}\vert b_j\vert\vert\wKK_I(\bm{w}_j\cdot \bm{x})-K_I(\bm{w}_j\cdot\bm{x})\vert\leq 2C_F\epsilon_0,
\end{align*}
which yields
\begin{align*}
\left\Vert\phi(\bm{x}) - f_{\epsilon_0}(\bm{x})\right\Vert_{L^2(\mu, B)}\leq 2C_F\sqrt{\mu(B)}\epsilon_0.
\end{align*}

The total approximation error satisfies
\begin{align*}
\left\Vert\phi(\bm{x}) - f(\bm{x})\right\Vert_{L^2(\mu, B)}\leq 4C_F\sqrt{\mu(B)}\epsilon_0.
\end{align*}
We take 
\begin{align*}
\epsilon_0=\frac{\epsilon}{4C_F\sqrt{\mu(B)}}
\end{align*}
to complete the proof.
\end{proof}

\subsection{Proof of Lemma~\ref{lem:NNex3}}
\begin{proof}
The proof of this lemma is simple by three facts: 1) the affine linear transforms before activation functions can play the role of translation and dilation in the spatial and Fourier domains; 2) the Gaussian activation function plays the role of localization in the transforms in this lemma; 3) Lem.~\ref{lem:NNex2} shows that the $x^2$ activation function can reproduce multiplication.
\end{proof}

\subsection{Proof of Lemma~\ref{lem:NNex4}}
\begin{proof}
The proof of this lemma is trivial by Lem.~\ref{lem:NNex2}, Thm.~\ref{thm:NNex2}, and the proof of Thm.~\ref{thm:analytic}.
\end{proof}

\section{Implementation details}
\subsection{Scientific computing}
The overall setting for all examples is summarized as follows.
\begin{itemize}
  \item \textbf{Environment.}
  The experiments are performed in Python 3.7 environment. We utilize PyTorch library for neural network implementation and CUDA 10.0 toolkit for GPU-based parallel computing. 
  \item \textbf{Optimizer.}
  In all examples, the optimization problems are solved by {\em Adam} subroutine from PyTorch library with default hyper-parameters. This subroutine implements the Adam algorithm in \cite{Kingma2014}.
  \item \textbf{Learning rate.}
 The learning rate will be decreased step by step in all examples following the formula
     \begin{equation}
     \tau_n=\tau_0 * q^{\floor{\frac{n}{s}}}, 
     \end{equation}
where $\tau_n$ is the learning rate in the $n-$th iteration, $q$ is a factor set to be $0.95$, and $s$ means that we update learning rate after  $s$ steps.
   \item \textbf{Numbers of samples.} The numbers of training and testing samples for regression and PDE problems are $10,000$. The numbers of training and testing samples for eigenvalue problems are $2048$ following the approach in \cite{HAN2020}.
  \item \textbf{Network setting.}
  In all PDE examples, we construct a special network that satisfies the given boundary condition as discussed in Section \ref{sec:LSM}.
In all examples, we apply ResNet with two residual blocks and each block contains two hidden layers. The width is set as $50$ unless specified. Unless specified particularly, all weights and biases in the $\ell$-th layer are initialized by $U(-\sqrt{N_{\ell-1}},\sqrt{N_{\ell-1}})$, where $N_{\ell-1}$ is the width of the $\ell-1$-th layer. Note that the network with RAFs can be expressed by a network with a single activation function in each neuron but different neurons can use different activation functions. For example, in the case of poly-sine-Gaussian networks, we will use $1/4$ neurons within each layer with $x$ activation function, $1/4$ with $x^2$, $1/4$ with $\sin(x)$, and $1/4$ with $\text{exp}(-x^2)$ for coding simplicity. In the case of poly-sine networks, $1/3$ neurons for each $x$, $x^2$, and $\sin(x)$ activation functions. In this new setting, it is not necessary to train extra combination coefficients in the RAF. Though training the scaling parameters in the RAF might be beneficial in general applications, we focus on justifying the poly-sine-Gaussian activation function without emphasizing the scaling parameters. Hence, in almost all tests in Part I, the scaling parameters are set to be one for $x$, $x^2$, and $\sin(x)$, and the scaling parameter is set to be $0.1$ for $\text{exp}(-x^2)$. In the case of oscillatory target functions, we specify the scaling parameter of $\sin(x)$ to introduce oscillation in the NTK as we shall discuss and improved performance is observed. The idea of scaling parameters has been tested and verified in \cite{JAGTAP2020109136,Ameya}. 
\item \textbf{Performance Evaluation.} We will adopt two criteria to quantify the performance of different activation functions. The first one is the relative $L^2$   error on test samples. Note that the ground truth solution is not available in real applications and, hence, it is not known when to stop the training. Therefore, we will keep the best historical $L^2$ test error and the best historical moving-average $L^2$ test error. In the moving-average error calculation, the error at a given iteration is the average $L^2$ test error of $100$ previous iterations. The second criterion is the condition number of the NTK matrices. A smaller condition number usually leads to a smaller iteration number to achieve the same accuracy.  
 \end{itemize}
}
\end{document}